\documentclass{article}
\usepackage[accepted]{tmlr}  
\usepackage{amsmath,amssymb,amsthm}
\usepackage{mathtools}
\usepackage{booktabs}
\usepackage{float}
\usepackage{graphicx}
\usepackage{xcolor}
\usepackage{tikz}
\usetikzlibrary{arrows.meta,positioning,fit,backgrounds,calc,shapes.geometric,decorations.pathreplacing}
\usepackage{hyperref}
\usepackage{url}

\definecolor{cblue}{RGB}{40,90,160}\definecolor{cblbg}{RGB}{225,235,248}
\definecolor{corange}{RGB}{225,140,40}\definecolor{corbg}{RGB}{252,238,220}
\definecolor{cgreen}{RGB}{40,150,80}\definecolor{cred}{RGB}{200,60,60}
\definecolor{cgray}{RGB}{120,120,120}
\definecolor{cgbg}{RGB}{226,242,231}\definecolor{crbg}{RGB}{250,231,231}
\usepackage{pifont}
\newcommand{\attgrid}[4]{%
  \begin{scope}[shift={(#1)}]
    \pgfmathsetmacro{\cs}{#4}
    \foreach \i in {0,...,\numexpr#2-1\relax}{
      \foreach \j in {0,...,\i}{
        \pgfmathtruncatemacro{\istar}{\i-#3}
        \ifnum\j=\istar \fill[cgreen!75] (\j*\cs,-\i*\cs) rectangle ++(\cs,-\cs);
        \else \fill[cblue!12] (\j*\cs,-\i*\cs) rectangle ++(\cs,-\cs);\fi
        \draw[white,line width=0.4pt] (\j*\cs,-\i*\cs) rectangle ++(\cs,-\cs);
      }}
  \end{scope}}
\hypersetup{colorlinks=true,linkcolor=blue,citecolor=blue,urlcolor=blue,
  pdftitle={A Positional Realizability Theory of Length Generalization in Single-Head Attention},
  pdfauthor={Jiasong Han, Yunkai Chai, Jintao Wang}}

\newtheorem{theorem}{Theorem}
\newtheorem{proposition}{Proposition}

\newtheorem{conjecture}{Conjecture}

\newcommand{\R}{\mathbb{R}}
\newcommand{\Phimat}{\Phi}
\newcommand{\Kn}{\mathcal{K}}

\newcommand{\alphaL}{\alpha_L}

\title{A Positional Realizability Theory of Length Generalization\\ in Single-Head Attention}

\author{\name Jiasong Han \email hanjiasongdhf@outlook.com \\
      \addr Department of Electronic Engineering, Tsinghua University
      \AND
      \name Yunkai Chai \email chaiyunkai1@gmail.com \\
      \addr School of Information and Communication Engineering, Shanghai Jiao Tong University
      \AND
      \name Jintao Wang \email wangjintao@tsinghua.edu.cn \\
      \addr Department of Electronic Engineering, Tsinghua University}

\def\month{10}  
\def\year{2026}
\def\openreview{\url{https://openreview.net/forum?id=DRwVtZF4b6}}

\begin{document}
\raggedbottom   
\maketitle

\begin{abstract}
Why does length generalization (LG) depend so sharply on the positional encoding (PE)? Relative
schemes extrapolate fixed offsets, absolute schemes fix positions, and neither handles growing offsets.
We give a \emph{realizability theory} for a deliberately \emph{controlled} abstraction, a single
positional-only attention head, that \emph{derives} this well-documented empirical pattern and says precisely where
it holds and breaks. The device is a \emph{simulation transfer operator} $T_L$ that re-evaluates a
selected length-$L_0$-realizing positional kernel on the longer grid, with a scalar $\alphaL\in[0,1]$ read off
\emph{before any training} from $(\Phimat, t_n)$ alone. Here $\alphaL$ serves as the computable \emph{witness} of that
taxonomy's derivation, not as a stronger empirical predictor. Our results are a \textbf{relative-head deep-row characterization} for kernels with a strict unique
global maximizer (Theorem~\ref{thm:growing}), an \textbf{absolute/relative duality} that derives the pattern
(Proposition~\ref{prop:duality}), a \textbf{learned-absolute underdetermination} result, and a
\textbf{static expressivity bound}; a sufficiency theorem (Theorem~\ref{thm:suff}) bridges $\alphaL$ to
LG, and the general deterministic case is a precise open conjecture (correcting a tempting but invalid
``outside-colspan'' argument). Experiments confirm the structure: in the controlled single-head regime
$\alphaL$ \emph{exactly separates} the taxonomy's generalizing and non-generalizing (PE, task) cells
($2625$ trained-model measurements; stable at vocab $64$/$d{=}128$). Together, these results turn the received PE
taxonomy into a \emph{positional-kernel transfer theorem}, and identify the content and depth mechanisms
that escape it. The theory concerns the positional channel alone: tasks whose targets depend on token
content, including the standard length-generalization suite such as addition and sorting, fall outside
Theorem~\ref{thm:growing}.
\end{abstract}

\section{Introduction}

Length generalization (LG), training on short sequences and testing on longer ones, is a sharp
probe of whether a transformer has learned an \emph{algorithm} rather than a length-specific shortcut,
and a practical bottleneck for tasks from arithmetic to long-context reasoning
\citep{anil2022exploring,jelassi2023length,zhou2024robust}. A large body of empirical work has distilled
a clear and consistent picture: LG depends acutely on the positional encoding (PE). Learned-absolute
PEs typically fail, relative schemes such as RoPE and ALiBi extrapolate \emph{fixed} relative offsets,
absolute schemes lock \emph{fixed} absolute positions, and \emph{growing}-offset tasks such as copy and
reverse defeat essentially everything
\citep{kazemnejad2023nope,press2022alibi,su2024roformer,ruoss2023randomized,wang2024nope}. Existing
theory captures pieces of this picture either qualitatively (the RASP-L conjecture,
\citealp{zhou2024rasp}; the C-RASP framework, \citealp{huang2025crasp}) or through bounds evaluated on
the trained model \citep{izzo2025simulation}. \textbf{What is missing is an account that \emph{derives}
this picture from the attention mechanism itself: a quantity, computable from the PE and the task
\emph{before any training}, that says which (PE, task) pairs can length-generalize and why.}

\paragraph{The quantity, in one example (Figure~\ref{fig:overview}).} Fit a single positional attention
head to the prev-token target at training length $L_0$: it attends from each position to the one before
it. Now re-evaluate the same kernel on a longer grid and count the fraction of new positions
whose attention still peaks on the previous token. For RoPE this fraction is $1$ (the relative kernel is
length-independent); for a copy task, whose target offset grows with the sequence, it is $0$. That
fraction is the \emph{simulation alignment} $\alphaL\in[0,1]$, the object we study. It is read off from
the PE's feature map and the task's target function alone, with \emph{no end-to-end Transformer training,
no data samples, and no learned Transformer parameters} (only a deterministic kernel-selection on the PE
basis). It is the computable witness of a \emph{simulation transfer operator} $T_L$ that re-evaluates the
selected length-$L_0$-realizing kernel on the length-$L$ grid (\S\ref{sec:operator}). It is
model-free and data-free, but \emph{selection-rule dependent}: the value depends on which realizing
kernel the $L_0$-fit selects (\S\ref{sec:operator}).

Such a pre-training, data-free quantity is useful in two ways. Practically, it lets
one triage (PE, task) choices before spending any compute. Scientifically, which is our focus and the
more durable contribution, it converts the empirical PE pattern from a list of observations into a
\emph{derived consequence} of what a single attention head can realize and transfer across lengths,
with theorems that say exactly where the pattern holds and proofs of where it must break.

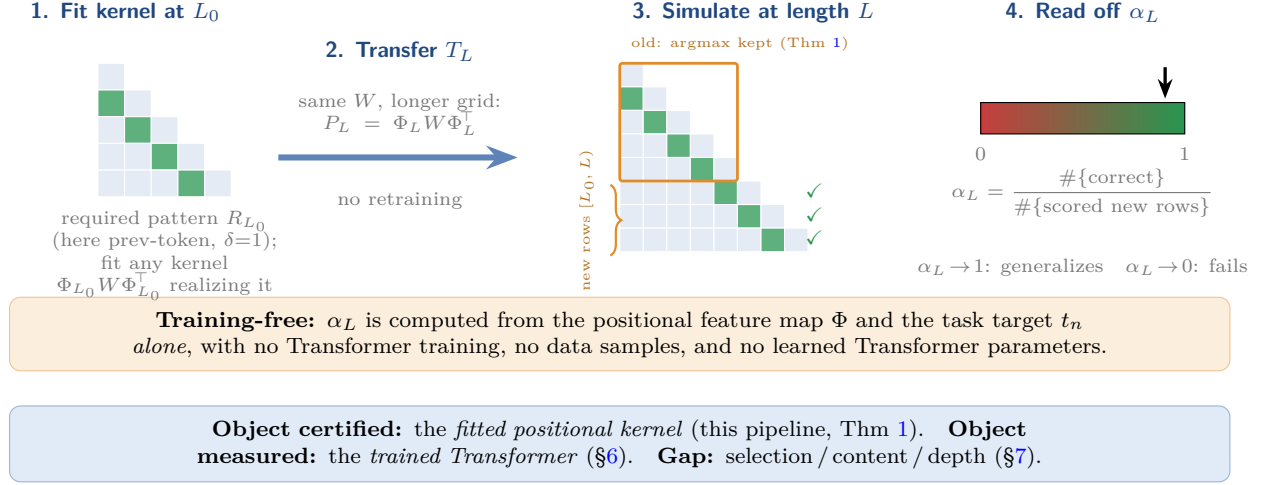
\begin{figure}[t]
\centering
\resizebox{\textwidth}{!}{%
\begin{tikzpicture}[font=\small,
   lbl/.style={font=\footnotesize\sffamily, text=cblue!80!black},
   note/.style={font=\scriptsize, text=cgray}]
\node[lbl] at (0,2.55) {\textbf{1. Fit kernel at $L_0$}};
\begin{scope}[shift={(-0.35,1.9)}]\attgrid{0,0}{5}{1}{0.34}\end{scope}
\node[note, text width=3.0cm, align=center] at (0.5,-0.55)
  {required pattern $R_{L_0}$ (here prev-token, $\delta{=}1$);\\ fit any kernel
   $\Phi_{L_0}W\Phi_{L_0}^{\!\top}$ realizing it};
\node[lbl, align=center] at (3.5,2.05) {\textbf{2. Transfer} $T_L$};
\draw[-{Stealth[length=4mm]}, line width=2.2pt, cblue!75] (1.95,0.7) -- (5.0,0.7);
\node[note, text width=3.0cm, align=center] at (3.5,1.25)
  {same $W$, longer grid:\\ $P_L=\Phi_L W \Phi_L^{\!\top}$};
\node[note, align=center] at (3.5,0.15) {no retraining};
\node[lbl] at (8.0,2.55) {\textbf{3. Simulate at length $L$}};
\begin{scope}[shift={(6.3,1.9)}]
  \attgrid{0,0}{8}{1}{0.30}
  \draw[corange,line width=1pt,rounded corners=1pt] (0,0) rectangle (1.5,-1.5);
  \node[note,text=corange!80!black,anchor=south west,font=\tiny] at (0.02,0.04) {old: argmax kept (Thm~\ref{thm:suff})};
  \draw[decorate,decoration={brace,amplitude=4pt},corange,line width=0.8pt] (-0.12,-1.55) -- (-0.12,-2.45);
  \node[rotate=90,note,text=corange!80!black,font=\tiny] at (-0.42,-2.0) {new rows $[L_0,L)$};
\end{scope}
\foreach \r in {0,1,2}{\node[font=\scriptsize] at (8.78, {0.28 - \r*0.30}) {\textcolor{cgreen}{\checkmark}};}
\node[lbl,align=center] at (12.2,2.55) {\textbf{4. Read off} $\alpha_L$};
\shade[left color=cred,right color=cgreen] (10.9,1.0) rectangle (13.5,1.4);
\draw (10.9,1.0) rectangle (13.5,1.4);
\node[note] at (10.9,0.75) {$0$};\node[note] at (13.5,0.75) {$1$};
\draw[-{Stealth[length=2.5mm]},line width=1.2pt] (13.25,1.85) -- (13.25,1.45);
\node[note,text width=3.4cm,align=center] at (12.2,0.25)
  {$\alpha_L=\dfrac{\#\{\text{correct}\}}{\#\{\text{scored new rows}\}}$};
\node[note,align=center] at (12.2,-0.7)
  {$\alpha_L\!\to\!1$: generalizes\quad$\alpha_L\!\to\!0$: fails};
\begin{scope}[on background layer]
  \node[rounded corners=4pt, fill=corbg, draw=corange!50, inner sep=5pt,
        text width=15.2cm, align=center, font=\footnotesize] at (6.3,-1.55)
    {\textbf{Training-free:} $\alpha_L$ is computed from the positional feature map $\Phi$ and the task
     target $t_n$ \emph{alone}, with no Transformer training, no data samples, and no learned Transformer parameters.};
  \node[rounded corners=4pt, fill=cblbg, draw=cblue!45, inner sep=5pt,
        text width=15.2cm, align=center, font=\footnotesize] at (6.3,-2.95)
    {\textbf{Object certified:} the \emph{fitted positional kernel} (this pipeline, Thm~\ref{thm:suff}).\quad
     \textbf{Object measured:} the \emph{trained Transformer} (\S\ref{sec:exp}).\quad
     \textbf{Gap:} selection\,/\,content\,/\,depth (\S\ref{sec:limits}).};
\end{scope}
\end{tikzpicture}}
\caption{\textbf{Overview of the method.} Select (fit) a single positional head's
kernel at length $L_0$ realizing the task's target pattern $R_{L_0}$; \emph{transfer} the same kernel to a
longer grid via $T_L$ (re-evaluate $\Phi_L W \Phi_L^\top$, no retraining); old rows keep their argmax by
the causal-mask identity (Thm~\ref{thm:suff}); score the new rows $[L_0,L)$ to read off
$\alpha_L\in[0,1]$. The whole pipeline uses only $(\Phi, t_n)$, with no Transformer training, no data
samples, and no learned Transformer parameters (only a deterministic kernel-selection on the PE basis).}
\label{fig:overview}
\end{figure}

We prove that $\alphaL=1$ implies exact length generalization
(Theorem~\ref{thm:suff}), and we characterize precisely when relative encodings reach it
(Theorem~\ref{thm:growing}): the required offset must settle on a single value on all but a vanishing
fraction of the ``deep'' rows of the attention grid, and that value must be attainable as a strict
maximizer inside the encoding's own kernel family. So prev/skip-$k$ generalize while copy and reverse
provably cannot (for kernels with a strict unique global maximizer), recovering the received taxonomy
as a theorem. Empirically, across $2625$ trained-model measurements, $\alphaL$ cleanly separates the
(PE, task) pairs that length-generalize from those that do not. We are deliberate about scope: we study a
\emph{controlled abstraction}, the positional channel \emph{in isolation}, precisely so the theory
can be exact. Content and depth escapes are explicitly mapped (\S\ref{sec:exp}--\ref{sec:limits});
e.g.\ an induction task length-generalizes through the content channel exactly where positional
$\alphaL{=}0$. We aim to characterize the positional channel's realizability and to locate where it
ceases to be the operative mechanism; explaining all of Transformer LG is outside our scope.

\paragraph{Contributions.}
\begin{enumerate}
\item \textbf{A realizability framework.} We cast LG as kernel transfer: a \emph{simulation transfer
operator} $T_L$ re-evaluates the selected length-$L_0$-realizing kernel at length $L$, and a scalar
$\alphaL$ (\S\ref{sec:operator}) reads off realizability from $(\Phimat, t_n)$ alone, using no end-to-end
Transformer training, no data samples, and no learned Transformer parameters (only a deterministic kernel-selection).
\item \textbf{Theory: a characterization of single-head positional LG.} The core results are a
\emph{deep-row characterization} for relative heads (Theorem~\ref{thm:growing}: among kernels with a
strict unique global maximizer, a relative family transfers a task iff the required offset is constant on
a density-one set of deep rows and that offset is attainable within the family, so prev/skip-$k$
generalize and copy/reverse do not), which
subsumes growing-offset impossibility and an \emph{absolute/relative duality}
(Proposition~\ref{prop:duality}) that \emph{derives} the empirical PE pattern; a \emph{learned-absolute
underdetermination} result (Proposition~\ref{prop:learnedabs}); and a \emph{static expressivity bound}
(Proposition~\ref{prop:static}). A \emph{sufficiency theorem} (Theorem~\ref{thm:suff}) is an elementary
but essential \emph{bridge} that makes $\alphaL$ meaningful: what it establishes is that the earlier,
already-seen rows keep the same masked attention pattern at longer lengths. Only the general deterministic-PE case remains
a precise \emph{conjecture} (Conjecture~\ref{conj:general}; we correct a tempting but invalid argument).
Several of the formal claims are accompanied by executable checks in the supplement; App.~\ref{app:repro}
lists which.
\item \textbf{Empirical confirmation and scope.} Across $2625$ trained-model measurements the
theory's structure holds: $\alphaL$ separates generalizing from non-generalizing (PE, task) cells with
AUC $=1.0$ (robust to vocab/width). We delineate the boundary precisely (\S\ref{sec:exp},
\S\ref{sec:limits}): where the \emph{content} channel carries the target (ALiBi prev, induction) LG
occurs without positional realizability, and $\alphaL$ under-predicts low-frequency relative PEs. Tasks
whose targets depend on token content, including the standard length-generalization suite such as
addition and sorting \citep{kazemnejad2023nope}, therefore fall outside Theorem~\ref{thm:growing}. As a
practical predictor $\alphaL$ \emph{matches but does not beat} the coarse taxonomy it explains (we
searched for a winning cell and report the negative); its contribution lies in the theory, and we do not
frame it as a benchmark comparison.
\end{enumerate}

\noindent\textbf{Organization.} \S\ref{sec:related} contrasts prior LG theory; \S\ref{sec:setup}--\ref{sec:operator}
set up the abstraction, the operator $T_L$, and $\alphaL$; \S\ref{sec:theory} proves the results;
\S\ref{sec:exp} confirms them empirically and maps the boundary where token content identifies the target; \S\ref{sec:limits}
states scope and limitations.

\section{Related Work}
\label{sec:related}
\textbf{Positional encodings for LG.} \citet{kazemnejad2023nope} benchmark PEs and find NoPE
competitive; \citet{wang2024nope} analyze why causal Transformers extrapolate without explicit PEs;
ALiBi \citep{press2022alibi} and RoPE \citep{su2024roformer} bias attention by relative position;
\citet{ruoss2023randomized} boost LG with randomized PEs, and \citet{cho2024position} via task-structured
position coupling; arithmetic and reasoning benchmarks drive much of this work
\citep{anil2022exploring,jelassi2023length,zhou2024robust}. These are empirical or architectural; we ask
the \emph{a priori} question of which (PE, task) pairs can extrapolate at all.
\textbf{Theory of LG.} RASP-L \citep{zhou2024rasp} conjectures LG holds iff a length-independent RASP
program exists, a discrete and qualitative criterion. \citet{izzo2025simulation} prove a simulation sufficiency
bound on the required \emph{training length}, evaluated on trained models, and explicitly leave open the
\emph{a priori} PE-ranking question. That is exactly our target, which we hit with a continuous PE-computable
quantity. \citet{huang2025crasp} give the C-RASP expressibility framework (task-level, PE-class-bound);
\citet{golowich2025sparsity} tie LG to attention sparsity. Against this line, our
Theorem~\ref{thm:growing} (relative kernels with a strict unique global maximizer cannot
length-generalize growing offsets) overlaps the
RASP-L/C-RASP intuition that copy/reverse are hard; its \emph{novel} content is the
\textbf{tie-rule-independent} deep-row argument and the \textbf{capacity corollary}
($\mathrm{acc_{train}}<1$ under any training-length distribution with unbounded support), neither of
which those frameworks state.
\textbf{Spectral views.} \citet{spectralpe2025} analyze PE spectra of \emph{trained} models and prove no
impossibility; our $T_L$ is a \emph{pre-training} operator yielding a sufficiency theorem, a
relative-encoding necessity theorem, and a (conjectural) general necessary condition.
\textbf{Expressivity, RASP, and circuits.} Our impossibility results sit in the lineage of attention
expressivity bounds: RASP \citep{weiss2021thinking} (which RASP-L and C-RASP build on), the
length-dependent self-attention limits of \citet{hahn2020theoretical}, the constant-depth threshold-circuit
view that grounds our static bound \citep{merrill2022saturated}, and the single-head capacity limits of
\citet{sanford2023representational}; the content/depth escape routes we note are the induction-head and
data-routing mechanisms of \citet{olsson2022induction,elhage2021mathematical,csordas2022neural}.
\textbf{Implicit bias.} $\alphaL$ is computed at the direction a regularized cross-entropy fit realizes
\citep{soudry2018implicit,gunasekar2018implicit}, which for attention is the token-selection / SVM
geometry of \citet{tarzanagh2023maxmargin,tarzanagh2023svm}; Theorem~\ref{thm:suff} is stated for any
realizing $W$, so it covers this object directly (no min-norm idealization).

\paragraph{Positioning.} Table~\ref{tab:position} situates $\alphaL$ against the
closest prior frameworks. The nearest is C-RASP[Pos] \citep{huang2025crasp}, which is also predictive
\emph{before} training; the differences are that it is \emph{task-level} and binary, tied to specific PE
classes, whereas $\alphaL$ is \emph{keyed to the PE's feature map} and continuous, computed at the
single-attention-head level from $(\Phimat,t_n)$. We do not claim to answer the \emph{a priori} question
first; C-RASP[Pos] does at the task level. What we add is finer resolution at the level of the PE feature
map, a continuous, not a binary verdict witness, and a kernel-transfer \emph{operator} that yields both the
sufficiency theorem and the relative growing-offset \emph{necessity} theorem (with its
tie-rule-independent capacity corollary) in one formalism. Theorem~\ref{thm:suff}'s treatment of the
already-seen rows is a single-head specialization of the simulation/transfer idea of
\citet{izzo2025simulation} (fit at $L_0$, re-evaluate at $L$); the new pieces are the closed-form
PE-computable witness $\alphaL$ and the necessity results. Theorem~\ref{thm:growing} in particular is a
relative-head necessity theorem with a deep-row characterization and a tie-rule-independent capacity
corollary, which goes beyond the qualitative observation that copy and reverse are hard. We do not claim
to be first to observe the empirical pattern; we derive it at the single-head positional level.

\begin{table}[h]
\centering
\caption{Positioning against the closest LG frameworks. $\alphaL$'s niche: pre-training,
\emph{keyed to the PE's feature map}, continuous, single-head, with matched sufficiency and relative-head
necessity (the general necessary condition is conjectural).}
\label{tab:position}
{\small\setlength{\tabcolsep}{4pt}
\begin{tabular}{lcccc}
\toprule
 & pre-training? & granularity & PE-keyed? & direction \\
\midrule
RASP-L \citep{zhou2024rasp}        & yes & discrete & no & suff.\ (conj.) \\
C-RASP[Pos] \citep{huang2025crasp} & yes & task, binary & by PE-class & suff.+nec. \\
Izzo et al.\ \citep{izzo2025simulation} & train-length bd. & model-level & no & sufficiency \\
spectral PE \citep{spectralpe2025} & post-hoc & model-level & spectral & -- \\
\textbf{$\alphaL$ (ours)} & yes & cell, \emph{cont.} & \emph{by feat.\ map} & suff.\ +\ rel.\ nec.\ +\ conj. \\
\bottomrule
\end{tabular}}
\end{table}

\section{Setup}
\label{sec:setup}
We follow the single causal attention head abstraction. A token $x_i$ at position $i$ has content
embedding $e(x_i)\in\R^d$; the PE supplies a positional feature map $\varphi:\mathbb{N}\to\R^m$. We
model the query--key logit as separable,
\begin{equation}
a_{ij} \;=\; \underbrace{e(x_i)^\top W_c\, e(x_j)}_{\text{content}} \;+\;
            \underbrace{\varphi(i)^\top W_p\, \varphi(j)}_{\text{position}},
\label{eq:logit}
\end{equation}
with $W_c\in\R^{d\times d}$, $W_p\in\R^{m\times m}$ learnable. We study \emph{position-driven} tasks:
for each query $i$ at length $n$, the head must attend to a task-determined target position $t_n(i)$
(e.g.\ prev-token $t=i-1$; copy $t=i-\lceil n/2\rceil$; reverse; attend-first $t=0$). The required
pattern is $R_n[i,j]=\mathbf{1}[j=t_n(i)]$, causal ($j\le i$). Scope hypotheses (S1)~single head, one
layer, positional-only; (S2)~``solve'' means correct masked row-argmax with margin $>0$;
(S3)~content-neutrality on evaluated rows. These are discussed and stress-tested in \S\ref{sec:limits}.

\paragraph{Realizable positional-kernel space.}
Stacking features $\Phimat_n\in\R^{n\times m}$ (row $i=\varphi(i)^\top$), the head's realizable
positional logits form the linear space $\Kn_n := \{\Phimat_n W \Phimat_n^\top : W\in\R^{m\times m}\}$,
of dimension $r_n^2$ with $r_n=\mathrm{rank}(\Phimat_n)$. NoPE has $r_n=1$ (constant kernel);
learned-absolute has $\varphi(i)$ undefined for $i\ge L_0$; sinusoidal/RoPE/ALiBi are deterministic at
all lengths, with RoPE/ALiBi depending only on $i-j$.

\section{The Simulation Transfer Operator and $\alphaL$}
\label{sec:operator}
A length-$L_0$-trained head is a parameter $W\in\R^{m\times m}$ whose positional kernel realizes the
target pattern $R_{L_0}$. Length generalization asks whether the same $W$ solves length $L$. We
therefore define the simulation transfer operator on the head's \emph{parameter} $W$ rather than as a map on
kernel \emph{blocks}. A block map would be ill-defined, since a block $\Phimat_{L_0}W\Phimat_{L_0}^\top$ does not
determine $W$ when $\Phimat_{L_0}$ is rank-deficient (then a nullspace direction $N$ has
$\Phimat_{L_0}N\Phimat_{L_0}^\top{=}0$ but $\Phimat_L N\Phimat_L^\top\ne0$, so two heads with the same
$L_0$ kernel transfer differently):
\begin{equation}
T_L(W) \;:=\; \Phimat_L W \Phimat_L^\top \qquad (\text{the head's kernel, re-evaluated on the length-}L\text{ grid}).
\end{equation}
A \emph{selection rule} $\mathcal{S}$ picks a realizing $W^{\mathcal S}$ from the $L_0$ fit. We use a
regularized cross-entropy fit on the kernel coefficients (400 Adam steps, weight decay $10^{-4}$), which
mirrors how an attention head is trained and respects the PE-specific constraints (RoPE relative kernel,
canonical ALiBi recency slope); as the regularization vanishes this direction approaches the max-margin
one \citep{soudry2018implicit}. The min-norm interpolant
$W^*=\Phimat_{L_0}^+ R_{L_0}(\Phimat_{L_0}^+)^\top$ is a different special case. The \emph{simulated pattern} is $P_L := T_L(W^{\mathcal S})$, and the \textbf{simulation alignment}
is
\begin{equation}\label{eq:alphadef}
\alphaL \;:=\; \frac{1}{|\mathcal{S}_L|}\sum_{i\in\mathcal{S}_L}
\mathbf{1}\!\left[\arg\max_{j\le i} P_L[i,j] = t_L(i)\right] \;\in\;[0,1],
\qquad
\mathcal{S}_L:=\{\text{query positions of the task}\}\cap[L_0,L).
\end{equation}
The scored set $\mathcal{S}_L$ is the task's own query region restricted to the new rows: all of
$[L_0,L)$ for prev/skip-$k$ and the fixed-position tasks, and the part of the second half that lies
beyond $L_0$ for copy and reverse, which is where their targets are defined. We use this one
normalization everywhere, including in Theorem~\ref{thm:growing} and in the implementation; averaging
instead over all of $[L_0,L)$ would score rows on which $t_L(i)$ does not exist and would change the
copy/reverse values by a factor that does not vanish with $L$. The argmax in \eqref{eq:alphadef} is
required to be \emph{strict}: a row counts as correct only if the target beats every competitor by a
positive margin, so an exact tie counts as a failure. This matters for the clamped learned-absolute
encoding, whose new-row logits tie exactly (Prop.~\ref{prop:duality}).
$T_L$, and hence $\alphaL$, is thus dependent on which realization is selected: different realizing $W$ transfer
differently, the existence-vs-selection gap we quantify in \S\ref{sec:exp}. The reported $\alphaL$ uses
the regularized cross-entropy section (400 Adam steps, weight decay $10^{-4}$); an ablation
(App.~\ref{app:add}) shows it tracks LG with $\rho{=}0.91$ versus
$\rho{=}0.56$ for the min-norm interpolant, which is why we adopt it. To be precise about what the
computation does and does not involve: the kernel-selection fit (the regularized cross-entropy section
above) is a deterministic
convex optimization over the PE's kernel coefficients on the $L_0$ constraint set. It involves
\emph{no end-to-end Transformer training, no data samples, and no learned Transformer parameters}, only
$\Phimat$, $t_n$, and the selection rule. We therefore call $\alphaL$ model-free and data-free, but
\emph{selection-rule dependent}, rather than ``gradient-free''.

\section{Theory}
\label{sec:theory}

\paragraph{One question behind the results.} The statements below are PE-specific, but they are all
instances of a single question. A positional encoding $\varphi$ induces a family of realizable kernels
$\mathcal{K}(\varphi)=\{\varphi(i)^\top W\varphi(j):W\}$, and a task demands, at row $i$ of length $n$, a
target key $t_n(i)$. Write $\delta_n(i)=i-t_n(i)$ for the required offset. Then
\begin{center}
\emph{does $\mathcal{K}(\varphi)$ contain a kernel whose masked argmax matches $\delta_n(i)$ on all but a
vanishing fraction of the new rows, at every length?}
\end{center}
Everything that follows answers this for one encoding class at a time. Theorem~\ref{thm:suff} says a
kernel that answers yes at length $L$ transfers, which is what makes $\alphaL$ meaningful.
Theorem~\ref{thm:growing} answers it for relative families, and the answer has exactly two parts: the
demanded offset must settle on a single value $\delta^*$ on a density-one set of deep rows, and $\delta^*$
must be attainable as a strict maximizer inside $\mathcal{K}(\varphi)$. Prop.~\ref{prop:static} shows the
affine ALiBi family fails the second part for every interior offset; Prop.~\ref{prop:learnedabs} shows
the learned-absolute family fails it on new rows because those coordinates are untrained;
Prop.~\ref{prop:sin} shows the sinusoidal family satisfies it yet the $L_0$-fit does not select the
kernel that does; and Prop.~\ref{prop:duality} reads off the resulting relative/absolute split.
Conjecture~\ref{conj:general} is the general form of the same question, for an arbitrary deterministic
$\varphi$, and it is the piece we do not have.

\begin{theorem}[Simulation sufficiency, for any realizing $W$]
\label{thm:suff}
Assume (S1)--(S3). Let $W\in\R^{m\times m}$ be any positional-kernel parameter such that
(i) $W$ solves length $L_0$ by argmax with margin $\gamma_0>0$, and (ii) the alignment computed at $W$,
$\alpha_L(W)$, equals $1$ with margin $\gamma_L>0$ on the new rows $[L_0,L)$. Then $W$ solves the task at
length $L$ with margin $\min(\gamma_0,\gamma_L)$.
\end{theorem}
\begin{proof}
New rows $i\in[L_0,L)$ are correct with margin $\gamma_L$ by (ii), since $\Phimat_L W \Phimat_L^\top$ are
exactly the logits $W$ produces at length $L$. For old rows $i<L_0$, the causal mask restricts
competitors to $j\le i<L_0$, so the length-$L$ logits there are \emph{identical} to the length-$L_0$
logits $\varphi(i)^\top W\varphi(j)$ (keys $j\in[L_0,i]$ do not exist when $i<L_0$); by (i) those rows
are correct with margin $\gamma_0$. Hence every row is correct.
\end{proof}
\noindent The proof uses no property of $W$ beyond (i)--(ii); in particular it does not
require the min-norm interpolant. This matters for the experiments: the reported $\alphaL$ is computed at
the regularized cross-entropy fit $W_{\mathrm{ce}}$, and Theorem~\ref{thm:suff} applies verbatim to $W_{\mathrm{ce}}$
\emph{whenever its $\alpha_L$ is $1$} -- so the theorem certifies the \emph{fitted kernel} we actually
compute . It does not certify the
trained transformer, a different object; how often $\alpha_L(W_{\mathrm{mm}}){=}1$ co-occurs with
trained-model LG is the empirical question of \S\ref{sec:exp}. Theorem~\ref{thm:suff} is intentionally
elementary: it isolates why the transfer score $\alpha_L$ is meaningful, namely that the earlier,
already-seen rows keep the same masked attention pattern at longer lengths, and the substantive open question is which $(\varphi,t)$ make the $L_0$-fit reach
$\alpha_L{=}1$ (Theorem~\ref{thm:growing}, Prop.~\ref{prop:sin} answer the negative side).

All three propositions below were obtained under independent adversarial verification (numeric
counterexample search $+$ proof audit); we state their \emph{corrected} cores and record every defect in
\S\ref{sec:limits}.

\begin{proposition}[Affine features force endpoint argmax]
\label{prop:static}
Assume (S1)--(S3) with $m=2$, $\varphi(i)=[1,i]^\top$. For every $W$ and row $i$, the logit
$a_{ij}=c_0(i)+c_1(i)\,j$ is affine in the key index $j$; hence the unique masked maximizer is the right
endpoint $j=i$ if $c_1(i)>0$, the left endpoint $j=0$ if $c_1(i)<0$, and a flat tie (margin $0$) if
$c_1(i)=0$. Thus \emph{no strict-interior key $0<j<i$ is ever the unique positive-margin argmax}, and any
interior-target task (e.g.\ prev-token, $i\ge2$) is unrealizable at every length $n\ge3$ for any $W$.
\end{proposition}
\noindent This is not a strawman: canonical ALiBi's positional basis is exactly affine in the offset,
$\{1,(i-j)\}$ with a sign-constrained slope, so Prop.~\ref{prop:static} is why ALiBi-prev has
$\alphaL{=}0$ in our experiments (\S\ref{sec:exp}) -- a single ALiBi head cannot place an interior
argmax.
\begin{proof}
$a_{ij}$ affine in $j\in\{0,\dots,i\}$ is monotone, so its maximizer is an endpoint, unique unless
$c_1(i)=0$. An interior target $0<t<i$ would need $a_{it}-a_{i0}=t\,c_1(i)>0$ and
$a_{it}-a_{ii}=(t-i)\,c_1(i)>0$, forcing $c_1(i)>0$ and $c_1(i)<0$ simultaneously -- a contradiction.
The competitor set is $\{0,\dots,i\}$ at every length, so the conclusion is length-independent.
\end{proof}

\begin{proposition}[Learned-absolute underdetermination]
\label{prop:learnedabs}
For learned-absolute PE with $\varphi(i)=v_i$ fixed for $i<L_0$ and i.i.d.\ absolutely-continuous for
$i\ge L_0$, and any non-degenerate selection rule $W=S(v_0,\dots,v_{L_0-1})$ depending on trained rows
only, the new-row argmax $M_i=\arg\max_{j\le i}v_i^\top W\varphi(j)$ has a law functionally
\emph{independent} of the new-key labels. Write $N_i=\{L_0\le j<i\}$ for the off-diagonal new keys and
$k=|\{L_0\le j\le i\}|$, and assume the row's target is one of them, $t_L(i)\in N_i$, with $|N_i|=k-1\ge1$.
Then
\begin{equation}\label{eq:prop2}
\Pr[M_i=t_L(i)\mid M_i\in N_i]=\frac{1}{k-1},
\qquad\text{and hence}\qquad
\Pr[M_i=t_L(i)]=\frac{\Pr[M_i\in N_i]}{k-1}\ \le\ \frac{1}{k-1}.
\end{equation}
Extrapolation is therefore at or below chance on such rows: the selection rule carries no information
about which untrained key is the target. The hypothesis $t_L(i)\in N_i$ is what makes the statement
about the untrained coordinates; it holds for the moving targets of prev and skip-$k$ once
$i\ge L_0+k$, which is the case the proposition is used for.
\end{proposition}
\begin{proof}[Proof sketch]
For new keys $g(j)=Wv_j$ depends only on the independent $v_j$; the map
$v_i\mapsto\arg\max_j v_i^\top g(j)$ is a fixed function in which the label $t_L(i)$ never appears, and
$S$ (trained-rows only) cannot encode it. Conditioned on $W$, the off-diagonal $g(j)$, $j\in N_i$, are
i.i.d.\ and independent of $v_i$, so their indices are exchangeable and the argmax is uniform on $N_i$
given that the maximum is attained there, which is \eqref{eq:prop2}. We assume as a hypothesis that the maximizer $M_i$ is almost surely unique. This does not follow from
$W\neq0$, nor from $W$ merely being injective on the compared keys: for an antisymmetric $W$ with
$Wv_0=0$ one has $v_i^\top Wv_i=v_i^\top Wv_0=0$ identically, so the self-key and the trained key $0$ tie
for every $v_i$ (App.~\ref{app:proofs}).
\end{proof}
\noindent\emph{Why both hypotheses are needed.} The probability in \eqref{eq:prop2} is conditional
because the trained (old) keys and the self-key $j=i$ are not exchangeable with the new keys:
the self-key contributes the quadratic $v_i^\top Wv_i$, so the global maximum need not lie in $N_i$ at
all. Monte-Carlo confirms both parts: at $L=30$ the conditional hit rate is $0.115$ against
$1/(k-1)=0.111$, while the unconditional rate is $0.034$ because the winner lands in $N_i$ only $30\%$ of
the time; at $L=60$ the figures are $0.025$ / $0.026$ / $0.016$ with $65\%$
(\texttt{experiments/verify\_prop2\_conditional.py}). The unconditional hit rate is therefore \emph{at or
below} chance. The hypothesis $t_L(i)\in N_i$ is equally necessary: if the target is an \emph{old} key
(as in attend-first, where $t_L(i)=0$), the trained keys are precisely the ones the $L_0$-fit has shaped,
and the hit probability can exceed $1/(k-1)$. With a min-norm $W$ fitted on the trained rows we
measure $\Pr[M_i=0]$ at up to $5.2\times$ the bound (feature dimension $m=64$, $L=128$; the ratio grows
with $m$ and $L$ and is below $1$ for small $m$; \texttt{experiments/verify\_prop2\_hypothesis.py}). That
regime is governed by Prop.~\ref{prop:duality} instead. Finally, a degenerate selection
such as $W=0$ makes every logit equal, which is why almost-sure uniqueness needs non-degeneracy.
\noindent\emph{This is underdetermination rather than an empty constraint cone:} a rescuing $W$ realizing the target on
new rows exists with positive probability ($\approx20$--$35\%$ in Monte-Carlo); the inductive-bias rule
is simply blind to the untrained coordinates and cannot select it.

\begin{proposition}[Sinusoidal: a realizing kernel exists but is not selected]
\label{prop:sin}
For sinusoidal absolute PE with incommensurate frequencies: (A)~for any \emph{fixed} offset $\delta$ a
$W$ realizing it exists at all lengths (the multi-frequency kernel has a unique global max at $d=\delta$,
disambiguated arbitrarily far beyond the lowest-frequency period); (B)~for \emph{growing} offsets
(copy/reverse), no translation-invariant kernel \emph{with a strict unique global maximizer} tracks a
length-dependent offset (Theorem~\ref{thm:growing}), so $\alpha_L\to0$; empirically the fitted kernel
gives $\alpha_L\approx0$ at every length we test.
\end{proposition}
\noindent\emph{The qualifier in (B) is necessary, and the sinusoidal family itself shows why.} A
translation-invariant kernel that is \emph{monotone increasing} over the accessible offset range has no
interior maximizer, and the causal window then hands it the largest available offset for free. Within the
sinusoidal span with the standard $16$ frequencies and base $10^4$, the single component
$\kappa(d)=\sin(d\theta_{\min})$ with $\theta_{\min}\approx1.78\cdot10^{-4}$ is increasing for
$d\lesssim8833$; at $L=17665$ the demanded copy offset is exactly $8833$ and this kernel attains
$\alpha_L=1$. This is far outside the lengths we study (our sinusoidal experiments reach $L=128$, where
the kernel is nowhere near its turning point), but it shows the impossibility must be stated for kernels
with a global peak, not for translation invariance alone.
\noindent\emph{In plain terms:} sinusoidal PEs are not impossible in principle for a fixed offset
since a realizing kernel exists at every length. What fails is \emph{selection}: the
training-length fit does not pick that kernel. We are therefore not claiming ``sinusoidal
fixed-offset cannot generalize''; we are isolating where it does and does not.
\noindent\emph{Existence $\ne$ selection.} The transfer operator's $L_0$-fit $W^*$ does not
recover the fixed-offset rescue: sinusoidal prev decays ($\alpha_L=0.45$ at $L{=}60$, $0.34$ at
$L{=}128$). The common intuition that a sinusoidal encoding must fail once the required offset exceeds
its lowest-frequency period is thus \emph{false}; the correct statement is conditional, and the sub-unit
transferred accuracy is exactly $\alpha_L$ measuring the existence-vs-selection gap.

We now give a genuine \emph{necessity} theorem for the most important failure case -- growing offsets
under relative encodings -- which explains the $\alpha_L{=}0$ copy/reverse cells of RoPE/ALiBi.

\begin{theorem}[Relative-head deep-row characterization]
\label{thm:growing}
Assume (S1)--(S3). Let $\mathcal{S}_L$ denote the scored rows (the task's query positions in $[L_0,L)$),
so that $\alpha_L$ is the fraction of $\mathcal{S}_L$ on which the masked argmax hits the target; we use
this one normalization throughout, matching the implementation. Fix a relative (translation-invariant)
family $\mathcal{K}$ of kernels $a_{ij}=\kappa(i-j)$, and consider those $\kappa\in\mathcal{K}$ that
admit a \emph{strict unique global maximizer} $c(\kappa)=\arg\max_{d\ge0}\kappa(d)$, under a fixed tie
rule. (A constant kernel has no strict maximizer, so NoPE, whose family is spanned by the constant, is
excluded.) Call a scored row \emph{deep} if $i\ge c(\kappa)$; the remaining scored rows form a set
$\mathcal{H}(\kappa)\subseteq[L_0,c(\kappa))$ of size at most $c(\kappa)-L_0$, a bound independent of $L$.
Assume the scored set grows, $|\mathcal{S}_L|\to\infty$, and write $\mathcal{K}_0$ for the kernels in
$\mathcal{K}$ that \emph{realize the task at $L_0$} with positive margin, which is the hypothesis
Theorem~\ref{thm:suff} needs in order to turn $\alpha_L$ into length generalization.
Then \textbf{some $\kappa\in\mathcal{K}_0$ with a strict unique global maximizer achieves
$\alpha_L\to1$ if and only if there is an offset $\delta^*$ with}
\begin{equation}\label{eq:thm2cond}
\text{\emph{(i) density:}}\quad
\frac{\#\{i\in\mathcal{S}_L:\ \delta_L(i)\neq\delta^*\}}{\#\mathcal{S}_L}\xrightarrow[L\to\infty]{}0,
\qquad\text{\emph{(ii) attainability:}}\quad \delta^*=c(\kappa)\ \text{for some }\kappa\in\mathcal{K}_0.
\end{equation}
For a given $\kappa$, $\alpha_L\to1$ iff \eqref{eq:thm2cond}(i) holds with $\delta^*=c(\kappa)$.
Both hypotheses are needed. Attainability must be inside $\mathcal{K}_0$ rather than $\mathcal{K}$: a
kernel peaked at the new rows' offset but failing at $L_0$ has $\alpha_L\to1$ without solving the task at
either length, so the existential quantifier alone would not deliver length generalization. And the proof
compares $|\mathcal{H}(\kappa)|$ against $|\mathcal{S}_L|$, which is informative only when the scored set
grows.
The exact conclusion $\alpha_L=1$ requires more than \eqref{eq:thm2cond}: the realized offset must hit the
target on every scored row, including the finitely many rows in $\mathcal{H}(\kappa)$, where the
causal window clips the argmax below $c(\kappa)$. Consequently fixed-relative-offset tasks (prev,
skip-$k$) are realizable by any family containing a kernel whose global maximizer is at $k$, while
growing-offset (copy) and within-length-varying (reverse) tasks are not: $\alpha_L\to0$ for every
$\kappa\in\mathcal{K}$ with a strict unique global maximizer.
\end{theorem}
\begin{proof}[Proof sketch (full proof, both tie rules, App.~\ref{app:proofs})]
\emph{Deep-row constancy of the realized offset:} for a deep row $i\ge c(\kappa)$ the causal window
$[0,i]$ contains $c(\kappa)$, and $c(\kappa)$ maximizes $\kappa$ over \emph{all} offsets $d\ge0$, hence
over the window; so the masked argmax offset equals the constant $c(\kappa)$, independent of $n$. (This
is why the maximizer must be global: a merely local peak could be beaten by a larger offset that the
window admits at long lengths.) A $\kappa$ therefore hits a deep row iff $c(\kappa)$ equals that row's
required offset, so
\begin{equation}\label{eq:alpha-count}
\alpha_L=\frac{\#\{i\in\mathcal{S}_L,\ i\ge c(\kappa):\ \delta_L(i)=c(\kappa)\}+r(\kappa)}{\#\mathcal{S}_L},
\qquad 0\le r(\kappa)\le|\mathcal{H}(\kappa)|,
\end{equation}
where $r(\kappa)$ counts the finitely many clipped rows that happen to be correct.
($\Leftarrow$) Given (ii), take $\kappa$ with $c(\kappa)=\delta^*$. By \eqref{eq:alpha-count} the deep
rows it misses are exactly those with $\delta_L(i)\neq\delta^*$, a vanishing fraction by (i), and
$|\mathcal{H}(\kappa)|/\#\mathcal{S}_L\to0$ since $\mathcal{H}(\kappa)$ is fixed while
$\#\mathcal{S}_L\to\infty$. Hence $\alpha_L\to1$.
($\Rightarrow$) Suppose $\alpha_L\to1$ for such a $\kappa$. Put $\delta^*:=c(\kappa)$, attainable by
construction, giving (ii); by \eqref{eq:alpha-count} the fraction of scored rows with
$\delta_L(i)\neq\delta^*$ is $1-\alpha_L+o(1)\to0$, giving (i). For the corollaries, no $\delta^*$
satisfies (i) when the offset \textbf{grows} (copy, $\delta_L=\lceil L/2\rceil$: for $L>2c(\kappa)$
every copy row is deep with realized offset $c(\kappa)$, so the missed fraction tends to $1$) or
\textbf{varies within a length} with positive density (reverse: a nested-window monotonicity argument
forces failure already at $L_0$).
\end{proof}
\noindent The hypotheses are each necessary, and the scope of the characterization is discussed with counterexamples in App.~\ref{app:thm2remarks}.

\begin{proposition}[Absolute/relative duality]
\label{prop:duality}
Assume (S1)--(S3). \emph{(a) Relative side.} A translation-invariant family containing a kernel whose
strict unique global maximizer is at $k$ length-generalizes the fixed relative offset $k$
(prev/skip-$k$: $\alpha_L{=}1$ at all lengths). On a fixed absolute position the demanded offset $i-t$
grows with the row, so by Theorem~\ref{thm:growing} no kernel \emph{with a strict unique global
maximizer} attains $\alpha_L\to1$; and empirically the kernel the $L_0$-fit selects decays,
$\alpha_L=0.80,0.40,0.20,0.12,0.055,0.039$ at $L=30,40,60,128,256,512$ for RoPE on attend-first.
\emph{(b) Absolute side, for the clamped learned-absolute encoding used in our experiments.} Let
$\varphi(i)=v_{\min(i,L_0-1)}$, so all positions $i\ge L_0$ share one feature. Such a head
length-generalizes a fixed \emph{absolute} position in the trained range ($\alpha_L{=}1$ at all lengths,
with an $L$-independent margin), and fails on a fixed relative offset: every new key carries the same
feature, so the logits over $j\ge L_0$ tie exactly and the target cannot be singled out
($\alpha_L{=}0$ with margin $0$). For the i.i.d.-extended variant the corresponding statement is
probabilistic and is Prop.~\ref{prop:learnedabs}. \emph{(c)} Neither side generalizes a growing offset
(copy/reverse), again for kernels with a strict unique global maximizer.
Empirically (App.~\ref{app:proofs}): attend-first clamped-absolute $1.00$/relative $0.00$; prev
clamped-absolute $0.00$ (an exact tie)/relative $1.00$. The i.i.d.-extended variant instead scores at the
chance level of Prop.~\ref{prop:learnedabs} ($\approx0.05$ in our Monte-Carlo), as it should: the two
objects fail for different reasons.
\end{proposition}
\noindent\emph{Which ``absolute'' object.} Stating part (b) for ``an absolute kernel'' without
qualification would conflate three distinct objects. \emph{(1) An unrestricted absolute kernel} $K[i,j]$, with a free entry for every
pair, is not limited to absolute targets at all: setting $K[i,i{-}1]$ large and the rest small
encodes the sub-diagonal directly, so it realizes prev-token at every length. \emph{(2) A learned-absolute
table extended by i.i.d.\ unseen rows} is at chance on new rows (Prop.~\ref{prop:learnedabs}).
\emph{(3) The clamped learned-absolute implementation} used in our experiments collapses all positions
$i\ge L_0$ onto one feature, so distinct new rows become indistinguishable and a moving target cannot be
resolved. The duality claim is about (2) and (3), the objects a trained model actually has, and part (b)
above now says so. The same care applies on the relative side: sinusoidal \emph{absolute} features can
represent translation-invariant fixed-offset kernels, as Prop.~\ref{prop:sin} itself shows, so
``absolute features'' and ``absolute-only behaviour'' are not the same thing. What separates the regimes
is the realizable kernel family a given encoding induces, which is exactly the quantity
condition~\eqref{eq:thm2cond}(ii) makes explicit.

\noindent\emph{Existence versus selection on the relative side.} It is tempting to say that a relative kernel simply \emph{cannot} reach a fixed absolute position because
that would need a peak at the growing offset $i$. That reasoning is wrong, because the causal window
truncates at $d=i$: a monotone increasing $\kappa$, e.g.\ $\kappa(d)=d$, puts the masked argmax at $j=0$
on every row at every length, so it realizes attend-first exactly. Such a kernel has no
strict unique global maximizer, which is why Theorem~\ref{thm:growing} does not cover it, and its
realizable analogue is present in genuine PE families: within the RoPE span, $\sin(d\theta_{\min})$ is
increasing over the whole range of lengths we study. What rules this out in practice is not
translation invariance but two other things: the canonical ALiBi slope is sign-constrained to be
non-positive (dropping that constraint does let an ALiBi-form kernel solve attend-first, while leaving
prev-token unrealizable as Prop.~\ref{prop:static} requires), and the $L_0$-fit does not select the
monotone solution, so the transferred accuracy decays with $L$ as reported in part~(a). This is the same
existence-versus-selection distinction we draw for sinusoidal encodings in Prop.~\ref{prop:sin}, and it
applies here too.

\begin{figure}[t]
\centering
\resizebox{0.92\textwidth}{!}{%
\begin{tikzpicture}[font=\small,
  hdr/.style={font=\footnotesize\sffamily\bfseries, text=cblue!85!black, align=center},
  rlbl/.style={font=\footnotesize\sffamily\bfseries, text=cblue!85!black, align=center},
  note/.style={font=\scriptsize, text=cgray, align=center, text width=2.85cm},
  cellY/.style={rounded corners=2pt, draw=cgreen!55, fill=cgbg, minimum width=3.2cm, minimum height=2.0cm},
  cellN/.style={rounded corners=2pt, draw=cred!45, fill=crbg, minimum width=3.2cm, minimum height=2.0cm}]
  \def\rA{0.8}\def\rB{-1.8}                      
  \node[hdr] at (3.4,2.25) {Fixed relative offset\\ \scriptsize $\delta=k$ (prev, skip-$k$)};
  \node[hdr] at (6.8,2.25) {Fixed absolute position\\ \scriptsize $t=0$ (attend-first)};
  \node[hdr] at (10.2,2.25) {Growing offset\\ \scriptsize $\delta_n=\lceil n/2\rceil$ (copy)};
  \node[rlbl] at (0.3,\rA) {Relative\\ $\kappa(i\!-\!j)$};
  \node[rlbl] at (0.3,\rB) {Clamped\\learned-abs.};
  \node[cellY] at (3.4,\rA){};\node[font=\Large,text=cgreen] at (3.4,{\rA+0.62}){\checkmark};
  \node[note] at (3.4,{\rA-0.42}) {peak $=$ target, all $L$ ($\alpha_L\!\to\!1$)};
  \node[cellN] at (6.8,\rA){};\node[font=\Large,text=cred] at (6.8,{\rA+0.62}){\ding{55}};
  \node[note] at (6.8,{\rA-0.42}) {needs offset $i$ (grows): no fit};
  \node[cellN] at (10.2,\rA){};\node[font=\Large,text=cred] at (10.2,{\rA+0.62}){\ding{55}};
  \node[note] at (10.2,{\rA-0.42}) {can't track $\lceil n/2\rceil$ (Thm~\ref{thm:growing})};
  \node[cellN] at (3.4,\rB){};\node[font=\Large,text=cred] at (3.4,{\rB+0.62}){\ding{55}};
  \node[note] at (3.4,{\rB-0.42}) {moving target: new keys share one feature};
  \node[cellY] at (6.8,\rB){};\node[font=\Large,text=cgreen] at (6.8,{\rB+0.62}){\checkmark};
  \node[note] at (6.8,{\rB-0.42}) {column-$0$ kernel (length-indep.)};
  \node[cellN] at (10.2,\rB){};\node[font=\Large,text=cred] at (10.2,{\rB+0.62}){\ding{55}};
  \node[note] at (10.2,{\rB-0.42}) {target drifts with $n$: no fit};
\end{tikzpicture}}
\caption{\textbf{The offset taxonomy, derived (Theorem~\ref{thm:growing}, Prop.~\ref{prop:duality}).}
A relative kernel length-generalizes a \emph{fixed relative offset}; a clamped learned-absolute encoding
does the reverse; and \emph{neither} handles a \emph{growing} offset (the shared blind spot, right
column). Throughout, the relative statements are for kernels with a strict unique global maximizer, the
hypothesis of Theorem~\ref{thm:growing}; the accompanying text explains why it is needed and what the
$L_0$-fit contributes. This $2\times3$ verdict is exactly the empirical PE taxonomy, here a consequence
of what a single positional head can realize.}
\label{fig:taxonomy}
\end{figure}

\begin{conjecture}[General length-generalization impossibility]
\label{conj:general}
For a deterministic PE defined at all lengths, if $R_L^{\mathrm{new}}\notin\Pi_{\Kn_L}$ and a
quantitative separation holds -- for every $W$ realizing $R_{L_0}$ with old-row margin $\ge\gamma$, the
residual $(I-\Pi_{\Kn_L})R_L^{\mathrm{new}}$ forces $\max_{j\le i,j\ne t}\varphi(i)^\top W\varphi(j)\ge
\varphi(i)^\top W\varphi(t_L(i))$ on some new row $i$ -- then no single head solves both $L_0$ and $L$.
\end{conjecture}
\noindent\emph{Remark (why the colspan argument is insufficient).} One might conclude impossibility from
``$\varphi(i)\notin\mathrm{colspan}(\Phimat_{L_0})$''. That does not suffice: the components of $W$ acting on
the orthogonal complement are \emph{unconstrained} by the $L_0$ fit and permit a rescuing $W$ as easily
as a failing one. The argmax flip is the substance of the statement and requires proof; we
leave it open.

\paragraph{Partial progress (App.~\ref{app:proofs}; the conjecture stays open).} We can say more.
\emph{(i) A saturation lemma.} For a finite-dimensional deterministic PE with $\Phimat_{L_0}$ of full
column rank $m$, the map $W\mapsto\Phimat_{L_0}W\Phimat_{L_0}^\top$ is injective
($\sigma_{\min}=\sigma_{\min}(\Phimat_{L_0})^2$), so a given kernel block determines $W$ and hence
$P_L$. \emph{(ii) But this does not decide LG}, because ``solving at $L_0$'' (S2) is an argmax+margin
\emph{cone} of $W$, not a single block: distinct blocks with identical $L_0$ argmax transfer to opposite
$L$-verdicts. LG is decidable only \emph{relative to a declared selection rule} (e.g.\ min-norm), not
from ``the $L_0$ fit'' alone -- the existence-vs-selection gap again. \emph{(iii) The obstruction is not
cone-emptiness}: an LP witness shows a single $W$ solving both $L_0$ and $L$ for prev-token sinusoidal in
the under-determined regime ($\mathrm{rank}=12\ll m=48$), so $C_{L_0}\cap C_L\ne\emptyset$; the failure
mechanism is selection (Prop.~\ref{prop:learnedabs}) rather than an empty cone. \emph{(iv) Sinusoidal escapes the
saturation route entirely}: at realistic frequency counts $\Phimat_{L_0}$ is rank-deficient or wildly
ill-conditioned ($\mathrm{cond}\sim10^{4}$--$10^{11}$), so $W$ is never robustly pinned. The sharpest
remaining open sub-case: a deterministic PE with a non-empty joint cone on a non-realizable
\emph{growing}-offset target and selection failure.

\section{Empirical confirmation of the theory}
\label{sec:exp}
We test whether trained transformers obey the realizability structure the theory predicts. These
experiments are meant to confirm the theorems rather than to benchmark a predictor: $\alphaL$ is the
theory's computable witness, and the relative/absolute taxonomy it should track is precisely the empirical
pattern the theory derives (Prop.~\ref{prop:duality}).

\begin{center}
\fcolorbox{cgray}{cgbg}{\parbox{0.95\linewidth}{\small
\textbf{What is certified vs.\ what is measured.} The theory certifies a \emph{fitted positional kernel}:
$\alphaL$ is a property of $(\Phimat,t_n)$, and Theorem~\ref{thm:suff} guarantees that a kernel that
realizes the task at $L_0$ with positive margin and reaches $\alphaL{=}1$ with positive margin on
the new rows still solves the task at length $L$; both premises are required. The \emph{trained Transformer} is a
\emph{separate} object, a non-convex optimizer over both positional and content channels. These
experiments test the empirical bridge: whether certified positional realizability ($\alphaL\!\to\!1$)
\emph{co-occurs} with the trained model's length generalization. A gap between them is therefore expected
and informative. It is exactly the selection/content/depth boundary of \S\ref{sec:limits}, and does not
contradict the theory.}}
\end{center}

\paragraph{Protocol.} Tiny decoder-only transformers (vocab $12$, $d=64$) trained at variable length up
to $L_0=20$ and evaluated on extrapolation rows $[L_0,L)$. Three depth tiers: the \emph{1-layer 1-head} setting the theory describes
(1-layer, 1-head) and two \emph{scale/escape} tiers (2-layer 4-head, 4-layer 4-head). Tasks span
bounded-offset (prev, skip-$k$), growing-offset (copy, reverse), and absolute (first) targets, across
five PEs (learned-abs, NoPE, ALiBi, sinusoidal, RoPE); $2625$ measurements
($3\times5\times7\times5\times5$) over $5$ seeds and $5$ test lengths. Where we report correlations on a
learnable subset ($\mathrm{train\ acc}>0.8$), it is because LG of a task a model never learned at $L_0$
is undefined; we always give the unconditioned numbers too. \textbf{The tasks are position-dependent
(small vocabulary, repeated tokens) so that content cannot localize the target position.} This
statistically discourages content shortcuts but does not make the trained model's content logit
literally argmax-neutral as (S3) assumes (the value/content path remains). So these experiments confirm
the positional theory \emph{within its assumed scope}; they do not confirm it independently of (S3).
Where content genuinely carries the target the impossibility results are vacuous, which is exactly the
boundary we map in \S\ref{sec:limits} and App.~\ref{app:content}.

\paragraph{The evaluation grid.} The taxonomy these experiments test is \emph{derived}
(Prop.~\ref{prop:duality}), so the axes follow the theory's own categories.
Table~\ref{tab:design} lists every condition we evaluate together with the reason for each level. The
task set covers all three offset regimes (fixed-relative, fixed-absolute, growing), and the PE set is
the five canonical classes rather than a subset.
\begin{table}[H]
\centering\small
\caption{The evaluation grid. Each axis follows the theory's categories or standard practice; the
reason given for each level is the reason it was included.}
\label{tab:design}
\begin{tabular}{@{}lll@{}}
\toprule
\textbf{Axis} & \textbf{Levels} & \textbf{Rationale} \\
\midrule
PEs    & learned-abs, NoPE, ALiBi, sinusoidal, RoPE & the five canonical PE classes \\
Tasks  & prev, skip-2/3/5, copy, reverse, first & fixed-relative, growing-offset, fixed-absolute regimes \\
       & ($+$ pos-2/5, App.~\ref{app:interior}) & interior absolute positions, off the causal boundary \\
Lengths& $30,40,50,60,80$ & extrapolation beyond $L_0{=}20$ ($1.5\times$--$4\times$) \\
Seeds  & $0,1,2,3,4$ & variance / bootstrap CIs \\
Models & 1L1H, 2L4H, 4L4H & 1-layer 1-head setting $+$ depth/multi-head escape routes \\
\bottomrule
\end{tabular}
\end{table}

\paragraph{Primary confirmation: $\alphaL$ separates generalizing from non-generalizing cells.} Over
\textbf{all $35$ (PE, task) cells} (no learnability filter), cell-level $\alphaL$ separates the $6$ that
length-generalize (LG~$\ge0.5$) from the $29$ that do not with \textbf{AUC $=1.00$} (base rate $0.17$;
permutation $p{=}5\times10^{-5}$, $20$k shuffles). This is not a fragile $1.0$: the margin is wide
(min $\alphaL$ among positives $=1.00$ vs.\ max among negatives $=0.36$), and restricting to the $11$
non-trivial cells (dropping negatives with $\alphaL{\approx}0$) keeps AUC $=1.0$. The screen is thus
validated \emph{on the hard cells too} (all $20$ NoPE/ALiBi/copy/reverse cells sit at $\alphaL{\approx}0$,
std $0.02$, and none generalize in this single-head regime). As a \emph{graded} (rank) statistic over the
same $35$ cells the unconditioned Spearman is modest ($\rho{=}0.47$), because the signal is dominantly
the clean binary separation; we report it for transparency and turn to the learnable subset for the
graded view.

\paragraph{$\alphaL$ acts primarily as a binary realizability indicator.} On the learnable subset
($\mathrm{train\ acc}>0.8$, which retains $15$ of $35$ cells; \textbf{all NoPE, ALiBi, and copy/reverse
are excluded}, so this subset can only probe the bounded/absolute regime), a naive cell-level
$\rho=0.90$ (bootstrap $95\%$ CI $[0.73,0.95]$, $15$ cells, $5$ seeds). But this is almost
entirely the binary $0$-vs-$1$ separation re-expressed (all numbers cell-level): binarizing $\alphaL$ at
$0.5$ already gives $\rho=0.85$, and restricting to non-trivial cells ($\alphaL>0.05$) gives $\rho=0.77$.
On the only family with genuinely interior $\alphaL$ values (sinusoidal bounded-offset) there is
\textbf{no positive graded signal}: $\rho=-0.20$ at the cell level ($n{=}4$ cells) and $\approx0$ at finer
aggregations ($0.04$ row-level, $0.09$ at $(\text{PE,task},L)$). We therefore do not claim $\alphaL$ is a
graded predictor; it is a \emph{binary realizability indicator}, which is what the theory predicts. (The unconditioned cell-level Spearman over all $35$ cells is $0.47$.)

\paragraph{Sufficiency is approximate, with characterized violations.} The
``$\ge92\%$-precision sufficient screen'' claim must be qualified. The precision of
$\alphaL\!\ge\!\tau \Rightarrow$ LG~$\ge0.7$ rises with $\tau$: $0.925$ at $\tau{=}0.5$, $0.95$ at
$\tau{=}0.7$, $0.987$ at $\tau{=}0.9$. At $\tau{=}0.5$ there are $12$ violating rows out of $160$, and they
come from three families: \texttt{rope-first} at $L{=}30$ ($\alphaL{=}0.80$, LG~$=0.13$, all $5$ seeds),
\texttt{sinusoidal-prev} at $L{=}40$ ($\alphaL{=}0.50$, LG~$=0.14$, all $5$ seeds), and
\texttt{sinusoidal-first} at $L{=}60,80$ ($\alphaL{=}1.0$, one seed each, the worst being LG~$=0.49$
against a $5$-seed cell mean of $0.77$). These are the
\emph{existence-vs-selection} gap of Prop.~\ref{prop:sin} surfacing in the \emph{validated} predictor:
the CE fit over-sharpens $\alphaL$ at small $L$ where the true selected kernel decays. We therefore do
not claim a clean row-level sufficient indicator; we claim a strong \emph{binary screen}
(AUC $1.0$) with a positive cell-level association on the learnable subset, with violations confined to
the relative/sinusoidal selection gap.

\paragraph{Relation to the received taxonomy.} A PE-\emph{blind} rule (``is the offset bounded?'')
attains \emph{negative} correlation ($\rho=-0.24$), so PE-awareness is essential. The hand-built taxonomy
baseline (relative-PE$\times$bounded $\lor$ absolute-PE$\times$absolute-task) reaches $\rho=0.85$, within
our bootstrap CI: $\alphaL$ \emph{matches but does not beat} it. \textbf{In the theory framing this is
the expected outcome rather than a shortfall:} that taxonomy is exactly the coarse rule
Proposition~\ref{prop:duality} \emph{derives}, so $\alphaL$ tracking it confirms the theory.
The role of $\alphaL$ is to be the directly computed witness of why
the taxonomy holds and \emph{where} it breaks (the boundary where token content identifies the target below); it is not offered as
a way to beat a baseline. We attempted to find
a PE the taxonomy cannot classify but $\alphaL$ can (a low-frequency RoPE, still ``relative''): the
attempt instead exposed an $\alphaL$ \emph{false negative}: on the content-neutral skip-$k$ task the
trained model generalizes a target the linear-kernel fit deems unrealizable. Since the task is
content-neutral this is not a content rescue in the sense of App.~\ref{app:content}; we have not isolated
the mechanism further and report it as an unexplained disagreement between the trained model and the fitted kernel, and as a negative
result for the comparison with the taxonomy.

\paragraph{What the theorem certifies vs.\ what we measure.} Theorem~\ref{thm:suff} (any
realizing $W$) certifies the \emph{fitted kernel} $W$ at which we compute $\alphaL$. It does not certify
the \emph{trained transformer}, which is a different parameter. So $\alpha_L(W_{\mathrm{fit}}){=}1$ guarantees
the fitted kernel solves length $L$; whether the trained model does is the empirical question, and the
two can diverge: \texttt{rope-first} ($\alphaL{=}0.80$, LG~$0.13$ at $L{=}30$) and
\texttt{sinusoidal-first} ($\alphaL{=}1.0$ at $L{=}80$, LG~$0.77$ on average and as low as $0.49$ on the
worst seed) are cases where $W_{\mathrm{trained}}\ne W_{\mathrm{fit}}$. We
therefore do not claim ``the certified object is the validated one.'' The right test of the
theorem's relevance is co-occurrence of $\alpha_L(W_{\mathrm{fit}}){=}1$ with trained-model LG, which
holds for $\ge92\%$ of high-$\alphaL$ rows (rising to $98\%$ at $\alphaL{\ge}0.9$); this is strong but not
exact. Theorem~\ref{thm:suff} covers any realizing kernel, so the cross-entropy fit we use is as certified as
the min-norm interpolant; the two simply differ in how well they predict the trained model, at
$\rho{=}0.91$ against $\rho{=}0.56$ (App.~\ref{app:add}). We therefore frame the empirical link as a
co-occurrence rate rather than a theorem.

\begin{table}[t]
\centering
\caption{Training-free $\alphaL$ (at $L{=}60$) vs.\ empirical LG accuracy (mean over $5$ seeds),
1-layer 1-head and scale (2L4H) tiers. Top: $\alphaL$ tracks 1L1H LG. Bottom: at scale, the
\emph{content/recency channel} (outside the positional-only abstraction, S3) lets ALiBi/NoPE generalize
where $\alphaL{=}0$, making the predictor's one-sidedness visible.}
\label{tab:main}
\begin{tabular}{llccc}
\toprule
PE & task & $\alphaL$ & LG (1L1H) & LG (2L4H) \\
\midrule
RoPE        & skip5  & 1.00 & $1.00{\pm}0.00$ & $0.97{\pm}0.02$ \\
RoPE        & prev   & 1.00 & $0.87{\pm}0.05$ & $0.85{\pm}0.13$ \\
RoPE        & copy   & 0.00 & $0.10{\pm}0.01$ & $0.14{\pm}0.01$ \\
sinusoidal  & first  & 1.00 & $0.87{\pm}0.13$ & $0.92{\pm}0.04$ \\
sinusoidal  & skip3  & 0.25 & $0.22{\pm}0.03$ & $0.11{\pm}0.01$ \\
learned-abs & first  & 1.00 & $1.00{\pm}0.00$ & $1.00{\pm}0.00$ \\
learned-abs & prev   & 0.00 & $0.08{\pm}0.00$ & $0.08{\pm}0.00$ \\
\midrule
\multicolumn{5}{l}{\emph{Content-channel cases (S3 violations): $\alphaL$ under-predicts at scale}}\\
ALiBi & prev  & 0.00 & $0.18{\pm}0.00$ & $\mathbf{1.00{\pm}0.00}$ \\
ALiBi & first & 0.00 & $0.15{\pm}0.01$ & $\mathbf{1.00{\pm}0.00}$ \\
NoPE  & first & 0.00 & $0.16{\pm}0.01$ & $\mathbf{1.00{\pm}0.00}$ \\
\bottomrule
\end{tabular}
\end{table}

\begin{figure}[t]
\centering
\includegraphics[width=\textwidth]{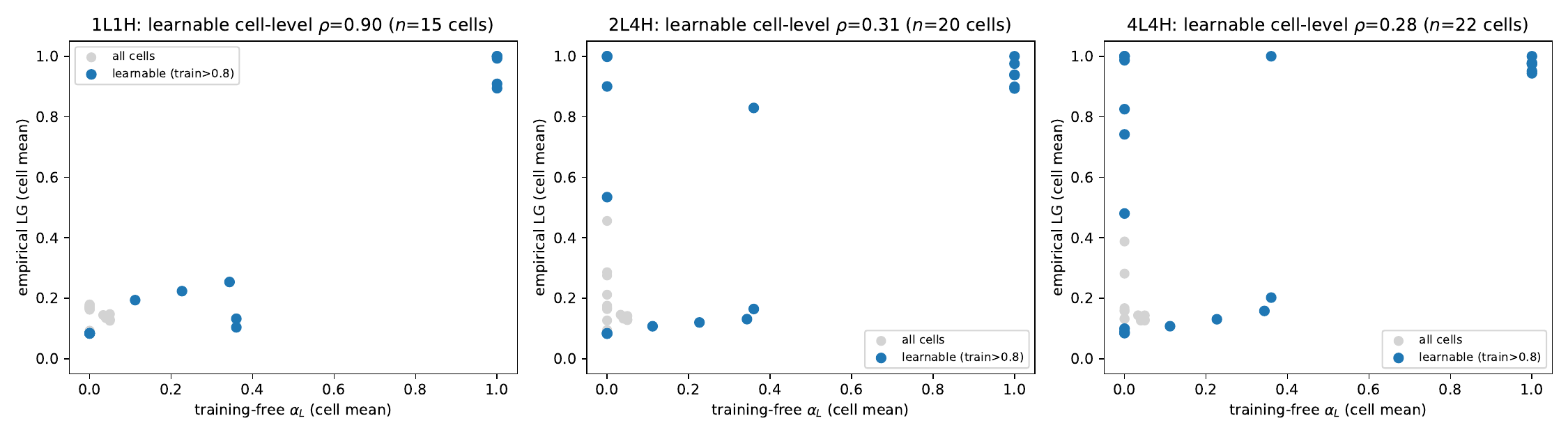}
\caption{Training-free $\alphaL$ vs.\ empirical LG across the three tiers ($5$ seeds). \emph{Left}
(1L1H, the setting the theory describes): learnable cells (blue) give Spearman $\rho{=}0.90$. \emph{Center, right}
(scale, 2L4H, 4L4H): the \emph{graded} rank correlation weakens (its bootstrap CI includes $0$) as the
content channel rescues low-$\alphaL$ cells, yet high $\alphaL$ remains a \emph{sufficient screen} at
every depth (precision $\ge0.92$).}
\label{fig:scatter}
\end{figure}

\begin{figure}[t]
\centering
\includegraphics[width=0.72\textwidth]{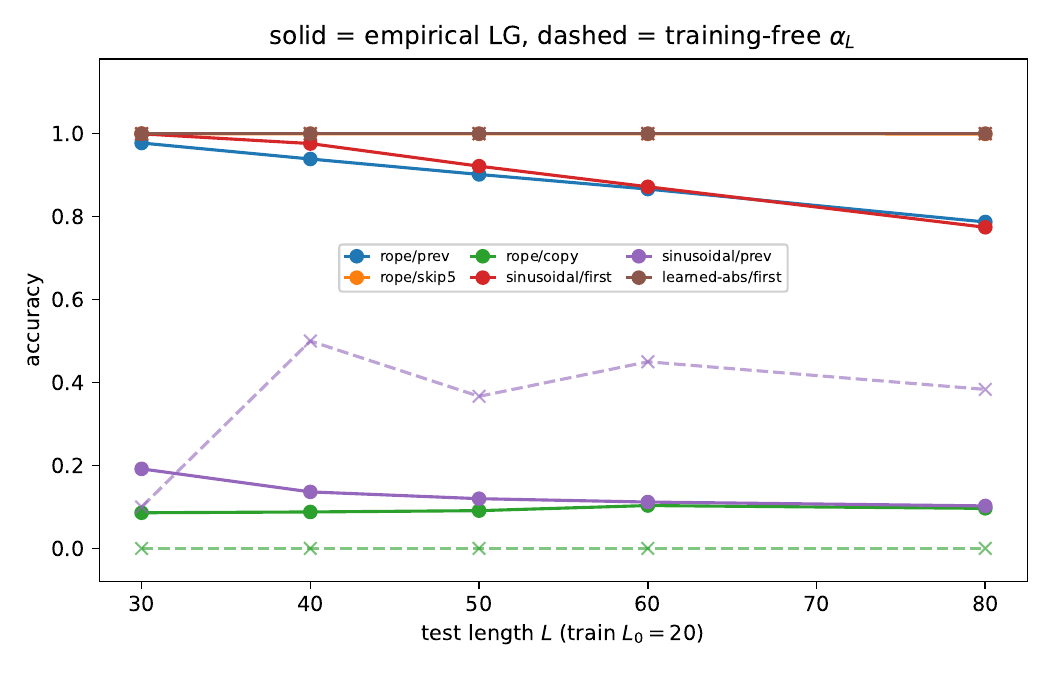}
\caption{Empirical LG (solid) and training-free $\alphaL$ (dashed) vs.\ test length $L$ (train
$L_0{=}20$), 1L1H. $\alphaL$ tracks both who generalizes and how LG decays with $L$.}
\label{fig:sweep}
\end{figure}

\paragraph{Scale and depth (H3): one-sidedness.}
At $2$- and $4$-layer/$4$-head, the graded rank correlation on learnable cells falls from $0.90$ to
$\rho=0.31$ and $0.28$, with bootstrap CIs that \emph{include $0$} (Table~\ref{tab:depth}), so beyond
depth $1$ $\alphaL$ is no longer a significant graded predictor. What \emph{does} survive at depth is the
\emph{sufficient}-screen property: precision rises to $0.94$ and $0.97$ (CIs $[0.89,0.97]$, $[0.94,0.99]$),
i.e.\ high $\alphaL$ still implies LG (Table~\ref{tab:depth}, Fig.~\ref{fig:depth}). The lost graded
signal is entirely attributable to
low-$\alphaL$ cells that generalize via mechanisms outside the positional-only abstraction: ALiBi and
NoPE realize prev-token and \texttt{first} through the content/recency channel ($\alphaL{=}0$, LG~$=1.0$;
Table~\ref{tab:main}, bottom), and the set of such cells \emph{grows} with depth. This is the predicted
(S3) boundary: $\alphaL$ is a \emph{sufficient} indicator rather than a necessary one, mirroring exactly that
sufficiency (Theorem~\ref{thm:suff}) and the relative growing-offset necessity
(Theorem~\ref{thm:growing}) are proven while the \emph{fully general} necessity
(Conjecture~\ref{conj:general}) is open. The heatmaps in Fig.~\ref{fig:heat} localize where content
rescues position.

\begin{figure}[t]
\centering
\includegraphics[width=\textwidth]{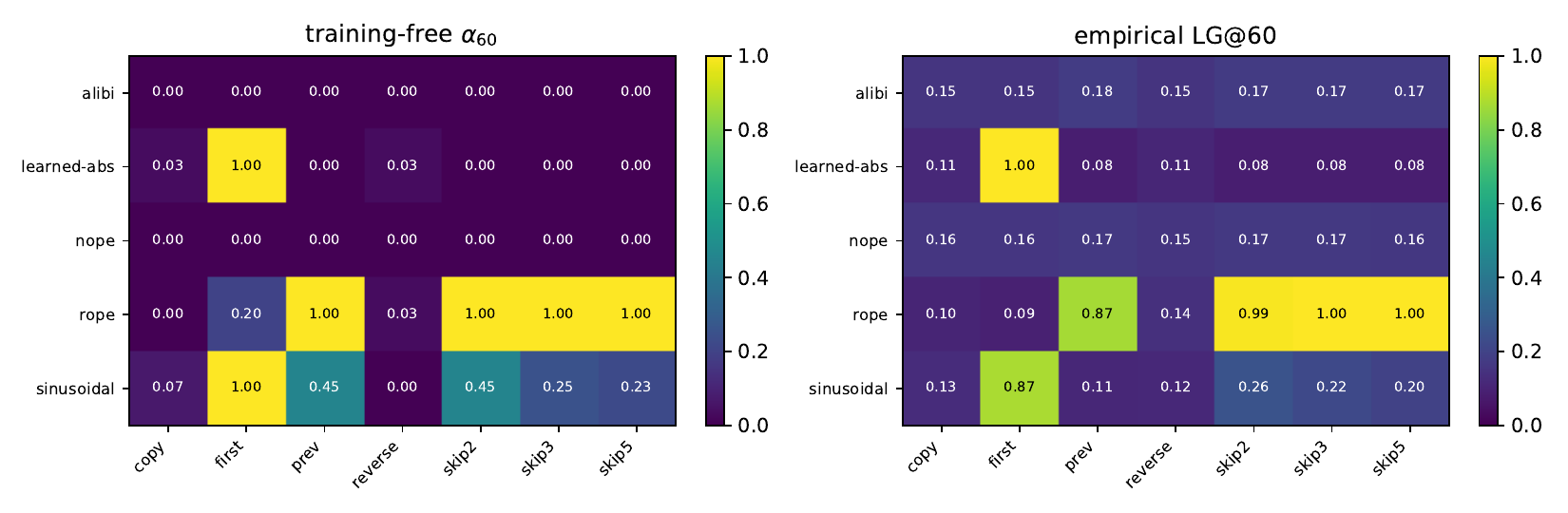}
\caption{Training-free $\alpha_{60}$ (left) and empirical LG@60 (right) over (PE, task), 1L1H. Agreement
is strong in this single-head regime; the content-driven disagreements (which emerge at depth, not shown
here) are characterized in \S\ref{sec:limits}.}
\label{fig:heat}
\end{figure}

\section{Limitations}
\label{sec:limits}

\noindent App.~\ref{app:failuremodes} sets out, with evidence, the five ways $\alphaL$ departs from a trained
model's behaviour: a selection gap, content rescue, a depth escape, low-frequency under-prediction, and
overprediction when the trained model selects a different kernel. The first four are failures of \emph{necessity} and leave the sufficiency direction intact, which is
why we call $\alphaL$ a one-sided screen; the fifth is the one case where high $\alphaL$ does not
co-occur with length generalization.

\medskip\noindent The remaining scope limits:
(1)~\textbf{Single head / single layer.} Multi-head and depth provably escape the impossibility (an
anchor head plus a content-match head solves copy at depth~1); all impossibility claims are scoped to
the positional-only subclass. (2)~\textbf{Content channel.} Where content carries the target, $\alphaL$
under-predicts; this is the binding limitation and the source of the indicator's one-sidedness. We show
it is closable (App.~\ref{app:content}): on a content-driven induction task, positional $\alphaL{=}0$ for
every PE yet a $2$-layer ALiBi model reaches LG~$0.84$, and a content-aware $\alpha_L^{+}$ flags the
regime $\alphaL$ misses (with residual PE-modulation left to future work).
(3)~\textbf{The forward question is open.} Theorem~\ref{thm:suff} (stated for any realizing $W$)
certifies that $\alpha_L(W){=}1\Rightarrow$ LG, but it does not say which $(\varphi,t)$ admit a
fitted $W$ reaching $\alpha_L{=}1$; we characterize the negative side (Theorems~\ref{thm:growing},
Prop.~\ref{prop:sin}) but a positive forward criterion is open.

\noindent A list of statements we explicitly do not claim, with the reason for each, is given in App.~\ref{app:notclaim}.

\section{Conclusion}
We gave a realizability theory of length generalization for single-head positional attention: a
simulation transfer operator and its scalar witness $\alphaL$, a sufficiency theorem, and a
\textbf{deep-row characterization for relative heads} (for kernels with a strict unique global maximizer,
$\alpha_L\to1$ exactly when the offset is constant
on a density-one set of deep rows and that offset is attainable within the encoding's kernel family) that
subsumes growing-offset impossibility and the absolute/relative duality
\emph{deriving} the empirical pattern, alongside a learned-absolute characterization; only the
general deterministic-PE case remains conjectural. Selected theoretical claims are accompanied by executable checks that probe their assumptions. Experiments confirm the structure:
in the controlled single-head regime $\alphaL$ \emph{exactly separates} the taxonomy's generalizing
and non-generalizing cells (AUC $=1.0$). They also locate the theory's boundary: the
content channel, outside the abstraction, can carry the target where positional realizability cannot
(ALiBi prev, induction), and $\alphaL$ under-predicts low-frequency relative PEs. Rather than replacing
the received taxonomy, $\alphaL$ \emph{derives} it from positional-kernel transfer and exposes its escape
routes. Extending the theory across the content channel and to multi-head/depth, which provably escape
the single-head impossibilities, is the natural next step.

\subsubsection*{Acknowledgments}
We thank the Action Editor, Zhanxing Zhu, and the three anonymous reviewers. Their careful reading
identified errors in earlier versions of Theorem~\ref{thm:growing} and Propositions~\ref{prop:learnedabs}
and~\ref{prop:duality}, and their comments substantially improved both the results and the presentation.

\subsubsection*{Use of language models}
Large language models (Claude and Codex) were used extensively in preparing this work: to draft and
edit the text, to write the experiment, analysis and verification code, and to develop and check the
mathematical arguments, including searches for counterexamples to the stated results. The authors take
full responsibility for the content. Every reported number can be reproduced from the released code and
data.

\bibliographystyle{tmlr}
\bibliography{refs}

\appendix
\section{Full proofs}
\label{app:proofs}

\subsection{Theorem~\ref{thm:suff} (Simulation sufficiency, any realizing $W$)}
Fix any $W$ satisfying (i)--(ii). Under (S1)--(S3) the row-$i$ logits at any length $n$ are
$a^{(n)}_{ij}=\varphi(i)^\top W\varphi(j)$ for $j\le i$, with the content term argmax-neutral.
\emph{New rows} $i\in[L_0,L)$: by (ii) the simulated $P_L[i,\cdot]=a^{(L)}_{i\cdot}$ has its unique masked
maximum at $t_L(i)$ with gap $\ge\gamma_L$. \emph{Old rows} $i<L_0$: the causal mask restricts the key
set to $\{0,\dots,i\}\subseteq\{0,\dots,L_0-1\}$, and for those $j$ the features $\varphi(j)$ are exactly
those present at length $L_0$; hence $a^{(L)}_{ij}=a^{(L_0)}_{ij}$ entrywise, so the masked argmax and gap
$\ge\gamma_0$ are inherited verbatim from (i). Every row's unique masked maximizer is $t_L(i)$ with gap
$\ge\min(\gamma_0,\gamma_L)>0$. \hfill$\square$

\emph{Remark (why not the min-norm interpolant).} Stating this for the min-norm $W^*$ would invoke
$\Phimat_{L_0}W^*\Phimat_{L_0}^\top=R_{L_0}$, which is false unless $R_{L_0}\in\Kn_{L_0}$. The argument
needs no such exactness and no choice of $W$: only that the chosen $W$ is argmax-correct at $L_0$ and has
$\alpha_L(W){=}1$. Hence it covers the fitted $W$ used in the experiments, though
not the trained transformer. The theorem's force comes entirely from the invariance of the already-seen
rows under the causal mask; the new-row clause restates (ii).

\subsection{Proposition~\ref{prop:static} (Affine features force endpoint argmax)}
With $\varphi(i)=[1,i]^\top$ write $W=\begin{psmallmatrix}W_{00}&W_{01}\\W_{10}&W_{11}\end{psmallmatrix}$.
Then $a_{ij}=\varphi(i)^\top W\varphi(j)=c_0(i)+c_1(i)\,j$ with $c_0(i)=W_{00}+iW_{10}$,
$c_1(i)=W_{01}+iW_{11}$, which is affine in the key index $j\in\{0,\dots,i\}$. An affine function on a
finite integer chain is monotone, so its maximizers are endpoints, and the maximizer is unique unless
$c_1(i)=0$ (a flat row, margin $0$, excluded by (S2)). For a strict interior target $0<t<i$,
optimality would require $a_{it}\ge a_{i0}$ and $a_{it}\ge a_{ii}$, i.e.\ $t\,c_1(i)\ge0$ and
$(t-i)c_1(i)\ge0$; with $t>0,\ t-i<0$ this forces $c_1(i)\ge0$ and $c_1(i)\le0$, hence $c_1(i)=0$ and a
zero margin. Therefore no interior key is ever the \emph{unique, positive-margin} maximizer. Since the
competitor set $\{0,\dots,i\}$ does not depend on $n$, the conclusion holds at every length; prev-token
($t=i-1$, interior for $i\ge2$) is thus unrealizable for all $n\ge3$. \hfill$\square$

\emph{Caveat.} The statement is about the unique positive-margin argmax, not the set inclusion
$\arg\max\subseteq\{0,i\}$, which fails on the measure-zero locus $c_1(i)=0$ (there the argmax set is
the whole chain).

\subsection{Proposition~\ref{prop:learnedabs} (Learned-absolute underdetermination)}
Let $\varphi(i)=v_i$ with $v_0,\dots,v_{L_0-1}$ fixed and $v_{L_0},\dots,v_{L-1}\sim\mu$ i.i.d.\ from an
absolutely continuous isotropic law, and $W=S(v_0,\dots,v_{L_0-1})$ any rule of the trained rows only
(so $W\perp\{v_j\}_{j\ge L_0}$). Put $g(j)=W\varphi(j)$. \emph{Uniqueness:} for two keys $j\ne j'$ with $j,j'\ne i$ the difference
$a_{ij}-a_{ij'}=v_i^\top W(v_j-v_{j'})$ is affine in the absolutely continuous $v_i$, and it is a nonzero
functional precisely when $W(v_j-v_{j'})\neq0$; a comparison against the self-key contributes
$v_i^\top Wv_j-v_i^\top Wv_i$, which is a polynomial of degree two in $v_i$ rather than an affine
functional. For pairs of distinct non-self keys the coefficient $W(v_j-v_{j'})$ is nonzero whenever $W$ does not
collapse them, and the zero set of a nonzero affine functional is null. Comparisons involving the
self-key are not covered by this argument, and no condition of that kind suffices: take $W$
antisymmetric with $Wv_0=0$. Then $v_i^\top Wv_i=0$ because $W$ is antisymmetric, and $v_i^\top Wv_0=0$
because $Wv_0=0$, so the self-key and the trained key $0$ carry \emph{identical} logits for every $v_i$,
while $Wv_j\neq Wv_i$ still holds generically. With positive probability the remaining logits are all
smaller and this identical pair is the maximum. We therefore take almost-sure uniqueness of $M_i$ as an
explicit hypothesis of the proposition rather than deriving it; it holds, for instance, whenever the
realized logit vector has distinct entries almost surely, which we verify numerically for the fitted
kernels we use. \emph{Target-independence:} the map $v_i\mapsto\arg\max_j v_i^\top g(j)$ is a
fixed measurable function of $(v_i,\{g(j)\})$ in which the label $t_L(i)$ never appears, and $S$ cannot
encode it into the new-key geometry; thus $\mathrm{Law}(M_i)$ does not reference $t_L(i)$.
\emph{Chance rate:} conditioned on $W$, the off-diagonal new keys $\{g(j):L_0\le j<i\}=\{Wv_j\}$ are
i.i.d.\ and independent of $v_i$, so by exchangeability their indices are equally likely to carry the
maximum given that the maximum is attained in $N_i$. With $t_L(i)\in N_i$ and
$k=|\{L_0\le j\le i\}|$ this gives $\Pr[M_i=t_L(i)\mid M_i\in N_i]=1/(k-1)$, and since
$\{M_i=t_L(i)\}\subseteq\{M_i\in N_i\}$, the chain rule gives
$\Pr[M_i=t_L(i)]=\Pr[M_i\in N_i]/(k-1)\le1/(k-1)$. The old keys and the self-key are not exchangeable
with $N_i$, so they may hold the maximum; that is why the unconditional value is the smaller one.
\hfill$\square$

\emph{Underdetermination, not cone-emptiness.} A rescuing $W$ realizing the target on all new rows
exists with positive probability (Monte Carlo: $\approx20$--$35\%$ of draws;
\texttt{experiments/\allowbreak theorem2b\_\allowbreak learned\_\allowbreak abs\_\allowbreak montecarlo.py}); the selection rule is simply blind to the
untrained coordinates. The self-key $j=i$ is excluded because $a_{ii}=v_i^\top W v_i$ is a
\emph{quadratic} (non-exchangeable) form with mean $\mathrm{tr}\,W$, which is why the conditional rate is
$1/(k-1)$ rather than $1/k$.

\subsection{Proposition~\ref{prop:sin} (Sinusoidal: existence vs.\ selection)}
\emph{(A) Existence.} With incommensurate geometric frequencies $\{\theta_k\}$, the translation-invariant
kernel $g(d)=\sum_k\cos(\theta_k(d-\delta))\le\#\{\theta_k\}$ with equality iff
$\theta_k(d-\delta)\in2\pi\mathbb{Z}$ for all $k$; incommensurability forces $d=\delta$, so $g$ has a
unique global maximum at the target offset for any fixed $\delta$ (verified numerically to
$\delta\approx5\times10^4$, far beyond the lowest-frequency period, refuting the intuition that a
sinusoidal encoding must fail once the offset exceeds that period). Hence a $W$ realizing a fixed offset exists at all lengths. \emph{(B) Growing
offsets.} A translation-invariant kernel \emph{with a strict unique global maximizer} $c(\kappa)$ realizes
the constant offset $c(\kappa)$ on every deep row (Theorem~\ref{thm:growing}), which can equal
$\delta_L=\lceil L/2\rceil$ only at $L\in\{2c(\kappa)-1,2c(\kappa)\}$, and at some $L\le c(\kappa)$ where
the window clips onto a lower local peak; for every $L>2c(\kappa)$ each scored row is
missed, so $\alphaL\to0$ on copy/reverse. (The earlier ``cannot hold at two lengths'' phrasing was
imprecise: shallow clipping does let one kernel solve copy at two \emph{particular} lengths, e.g.\
$\kappa[10]{=}5,\kappa[30]{=}8$ at $L_0{=}20$ and $L{=}60$; the correct statement is asymptotic. Without
a global maximizer the conclusion fails outright, as the monotone example after the proposition shows.) \emph{Existence $\ne$ selection:} the $L_0$-fit $W^*$ does not recover the
fixed-offset rescue -- under the transfer operator sinusoidal prev decays ($\alphaL=0.45$ at $L{=}60$,
$0.34$ at $L{=}128$), and the sub-unit value \emph{is} $\alphaL$ quantifying the selection gap.
\hfill$\square$

\subsection{Theorem~\ref{thm:growing} (Relative-head deep-row characterization)}
Let $a_{ij}=\kappa(i-j)$ on causal keys $j\le i$, with a strict unique \emph{global} maximizer
$c(\kappa)=\arg\max_{d\ge0}\kappa(d)$, under a fixed tie rule. Write $\mathcal{S}_L$ for the scored rows
(the task's query positions in $[L_0,L)$), which is the set $\alpha_L$ averages over.
\emph{Deep-row constancy (Lemma).} For any row $i\ge c(\kappa)$ the window $[0,i]$ contains $c(\kappa)$,
and no admitted offset beats it because $c(\kappa)$ is the maximum over all $d\ge0$; so the masked argmax
offset is the constant $c(\kappa)$ for all such $i$ and all $n$ (verified over $4000$ random kernels, both
tie rules; globality is what makes this work for the unbounded-support RoPE and ALiBi families).

By the Lemma, $\kappa$ hits a deep scored row $i$ iff $\delta_L(i)=c(\kappa)$, so with
$\mathcal{H}(\kappa)=\mathcal{S}_L\cap[L_0,c(\kappa))$, whose size is at most $c(\kappa)-L_0$ uniformly in
$L$ (it is empty once the query region starts above $c(\kappa)$, as happens for copy and reverse),
\begin{equation}\label{eq:alpha-density}
\alpha_L \;=\; \frac{\#\{i\in\mathcal{S}_L,\ i\ge c(\kappa):\ \delta_L(i)=c(\kappa)\}+r(\kappa)}{\#\mathcal{S}_L},
\qquad 0\le r(\kappa)\le|\mathcal{H}(\kappa)|,
\end{equation}
so the clipped rows contribute $O(1/\#\mathcal{S}_L)\to0$. Note that both the theorem and
\eqref{eq:alpha-density} normalize by $\#\mathcal{S}_L$; this is the same convention as the
implementation, and it matters, since for tasks whose query region is a strict subset of $[L_0,L)$ (copy,
reverse) the two would differ by a non-vanishing factor.

\emph{($\Leftarrow$)} Assume \eqref{eq:thm2cond}: the required offset equals $\delta^*$ off a
density-zero subset of $\mathcal{S}_L$, and some $\kappa\in\mathcal{K}_0$ has $c(\kappa)=\delta^*$. Take
that $\kappa$. Old rows are correct because $\kappa$ realizes the task at $L_0$, and the causal-mask
identity of Thm~\ref{thm:suff} preserves that correctness at length $L$; by
\eqref{eq:alpha-density} the missed deep rows form a vanishing fraction and $|\mathcal{H}(\kappa)|$ is
fixed, so $\alpha_L\to1$. (Verified for $\delta^*\in\{1,2,3,5,8\}$, $L$ up to $5000$.) Note both
hypotheses are needed: for the canonical ALiBi family $\kappa(d)=c_0+sd$, $s\le0$, the maximizer is
$c(\kappa)=0$ for every admissible slope, so no ALiBi kernel realizes $\delta^*\ge1$ however constant the
demanded offset is.

\emph{The exact conclusion $\alpha_L=1$ needs every scored row, not just the deep ones.} The clipped set
$\mathcal{H}(\kappa)$ never leaves $\mathcal{S}_L$, so if a row in it is missed, $\alpha_L<1$ at every
$L$. Example: $\kappa(d)=-|d-30|$ with $L_0=20$ leaves the ten rows $[20,30)$ clipped; a task demanding
offset $1$ there and $30$ on deep rows satisfies exact constancy on all deep rows, yet
$\alpha_L=1-10/(L-20)$, i.e.\ $0.500,0.875,0.990,0.99995$ at $L=40,100,10^3,2\cdot10^5$: it tends to
$1$ but never equals it. Hence $\alpha_L=1$ is equivalent to the realized offset hitting the target on
every scored row, deep and clipped alike.

\emph{($\Rightarrow$)} Suppose $\alpha_L\to1$ for some $\kappa\in\mathcal{K}_0$. Put $\delta^*=c(\kappa)$,
which is attainable in $\mathcal{K}_0$ by construction, giving \eqref{eq:thm2cond}(ii). By
\eqref{eq:alpha-density} the fraction of scored rows with $\delta_L(i)\neq\delta^*$ equals
$1-\alpha_L+o(1)\to0$, giving \eqref{eq:thm2cond}(i). For the corollaries, no $\delta^*$ can satisfy
\eqref{eq:thm2cond}(i) in either of two cases. \textbf{(a) The offset grows with $n$}
(copy, $\delta_n=\lceil n/2\rceil$): for all $L>2c(\kappa)$ every copy row is deep with realized offset
$c(\kappa)$, which cannot equal $\lceil L/2\rceil\to\infty$, so the failing fraction tends to $1$ and
$\alpha_L\to0$. \textbf{(b) It varies within a length} (reverse, $\delta_n(i)=2i-2\lceil n/2\rceil+1$,
distinct offsets $\{1,3,5,\dots\}$): a nested-window monotonicity argument gives a contradiction already
at $L_0=20$. Row $18$ demands offset $17$ and its window is $\{0,\dots,18\}$, which contains offset $15$,
so $\kappa[17]>\kappa[15]$. Row $17$ demands offset $15$ and its window $\{0,\dots,17\}$ contains offset
$17$, so $\kappa[15]>\kappa[17]$. The two are incompatible, independent of $c(\kappa)$.
A $20$k-kernel sweep confirms no $\kappa$ solves reverse at $L_0$ (\texttt{iff\_constant\_offset\_lg.py}). \hfill$\square$

\emph{Why density and not exact constancy.} See App.~\ref{app:thm2remarks}(a) for the counterexample that forces condition~\eqref{eq:thm2cond}(i) to be stated in density.

\emph{Why deep-row/asymptotic, not a finite-$L$ ``complete iff''.} One might hope to add the necessity
clause ``shallow rows $=\min(i,\delta^*)$'', but it is \textbf{false}: $\kappa=[1.46,1.96,1.80,6.32]$ (offsets $0$--$3$, $c(\kappa)=3$) makes deep rows take
offset $3$ but shallow row $i{=}2$ argmax offset $1\ne\min(2,3){=}2$ (window values $[1.46,1.96,1.80]$).
Shallow rows are a finite, $\kappa$-dependent set, uncharacterized by a universal formula but immaterial
to $\alpha_L$ as $L\to\infty$ -- hence the criterion is on deep rows. Two further cautions: the
growing-offset clash is \emph{asymptotic} rather than two-length ($\kappa[10]{=}5,\kappa[30]{=}8$ solves copy at
both $L_0{=}20$ and $L{=}60$ via shallow clipping); and ``bounded offset'' is insufficient (the $i$-parity
offset $\delta(i)\in\{1,2\}$ scores exactly $0.5$, $20$k-kernel sweep). \emph{Capacity corollary:} for a training-length distribution with unbounded support, every fixed
kernel with a strict unique global maximizer incurs positive error on copy, so
$\mathrm{acc}_{\text{train}}<1$. The restriction matters: a kernel peaked at $c$ solves copy at
$L\in\{2c-1,2c\}$, so a distribution supported on those two lengths alone can be solved exactly.

\subsection{Theorem~\ref{thm:growing}: necessity of the hypotheses, and scope}
\label{app:thm2remarks}
\noindent\emph{Why each hypothesis is needed.} \emph{(a) The condition is asymptotic.} Requiring
$\delta_L(i)=\delta^*$ on every deep row would be too strong: $\alpha_L$ only measures a \emph{fraction},
so a density-zero exceptional set is compatible with $\alpha_L\to1$. Concretely, let $\delta_L(i)=2$ when
$i$ is a perfect square and $1$ otherwise. A kernel peaked at offset $1$ fails exactly on the squares, so
$1-\alpha_L=O(L^{-1/2})\to0$ even though the offset is not eventually constant
($\alpha_L=0.944,0.987,0.996$ at $L=200,5\cdot10^3,5\cdot10^4$).
\emph{(b) Attainability inside the family.} Constancy alone is not sufficient: $\delta^*$ must be
attainable in $\mathcal{K}$, and one cannot simply exhibit $\kappa(d)=-|d-\delta^*|$ without checking that
it belongs to the PE-induced family. Canonical ALiBi is exactly this case: its family
$\kappa(d)=c_0+sd$ with $s\le0$ has its maximizer at offset $0$ for every admissible slope, so no ALiBi
kernel peaks at $\delta^*{=}1$ and ALiBi cannot realize prev-token positionally, consistent with
Prop.~\ref{prop:static}. Condition~\eqref{eq:thm2cond}(ii) makes this
explicit. \emph{(c) The existence of a global maximizer is load-bearing.} Translation invariance alone does not
forbid tracking a \emph{growing} offset, because the causal window admits ever larger offsets as $L$
grows. Take the unbounded ladder $\kappa(d)=-1$ except $\kappa(20\cdot4^m)=m+2$ for $m=0,1,2,\dots$; its
values increase without bound, so it has no maximizer at all, and it attains $\alpha_L=1$ on copy along
the whole sequence $L=40,160,640,2560,10240,40960$ (we verified this numerically). A kernel that does
have a strict unique global maximizer $c(\kappa)$ can also solve copy at particular lengths, namely at
$L\in\{2c(\kappa)-1,2c(\kappa)\}$, where the demanded offset $\lceil L/2\rceil$ equals $c(\kappa)$, and
possibly at some $L\le c(\kappa)$ where the causal window clips the argmax onto a lower local peak (for
$\kappa$ peaked at $100$ with a secondary peak at $20$ we find $\alpha_L=1$ at $L=39,40,199,200$), but
it fails for every $L>2c(\kappa)$, which is the content of the corollary. The impossibility is
asymptotic; finitely many lucky lengths are not a counterexample to it. Growing-offset impossibility is therefore a statement
about kernels with a global peak, and we state it with that qualifier throughout. \emph{(d) Deep rows are
defined by $c(\kappa)$ alone.} One might also ask for $i\ge D$ for a support radius $D$, but that would be
vacuous for the very families we care about, since RoPE and ALiBi kernels have unbounded support, so
$D=\infty$ and no row would be deep. Globality of the maximizer is what the argument actually needs, and
with it the clipped set is $\mathcal{H}(\kappa)\subseteq[L_0,c(\kappa))$.
\noindent\emph{Scope of the characterization.} The criterion is
on \emph{deep} rows and the asymptotic $\alpha_L$; we do not claim a finite-$L$ ``complete iff''
that also pins shallow rows to $\min(i,\delta^*)$ -- that is \textbf{false}: the kernel
$\kappa=[1.46,1.96,1.80,6.32]$ (peak at offset $3$) makes shallow row $i{=}2$ select offset $1$, not
$\min(2,3){=}2$ (we reproduce this in App.~\ref{app:proofs}). Shallow rows are a finite, $\kappa$-dependent
set that does not affect the asymptotic $\alpha_L$. Two further cautions: the growing-offset clash is
\emph{asymptotic} (a kernel peaked at offset $30$ solves copy at both $L_0{=}20$ and $L{=}60$ via shallow
clipping -- there is no two-length contradiction); and ``bounded offset'' is insufficient. The last point
is now exactly condition~\eqref{eq:thm2cond}(i): an $i$-parity offset alternating in $\{1,2\}$ is bounded,
but each value occurs with density $1/2$, so no $\delta^*$ has a vanishing exceptional fraction and
$\alpha_L$ sits at $0.5$. What matters is that one offset occurs with density one; boundedness alone is not enough. This theorem subsumes the
growing-offset impossibility and the absolute/relative duality. Capacity corollary: for a training-length distribution with unbounded support, every fixed
kernel with a strict unique global maximizer incurs positive error on copy, so
$\mathrm{acc}_{\text{train}}<1$. The restriction matters: a kernel peaked at $c$ solves copy at
$L\in\{2c-1,2c\}$, so a distribution supported on those two lengths alone can be solved exactly. We
additionally stress-tested the corrected statement numerically, fuzzing both the kernel
($4000$ random kernels under both tie rules, in \texttt{iff\_constant\_offset\_lg.py}) and the task's offset function (constant, perfect-square,
power-of-two and cube exceptional sets, parity, drifting and block-alternating offsets, and growing
offsets), confirming that $\alpha_L\to1$ tracks density-1 constancy and nothing weaker
(\texttt{experiments/verify\_thm2\_density.py}); the proof does not rely on these checks.

\subsection{Proposition~\ref{prop:duality} (Absolute/relative duality)}
\emph{Relative} $\kappa(i-j)$: a fixed offset $k$ is realized by a $\kappa$ whose strict unique global
maximizer is $d{=}k$, correct on every row at every length ($\alphaL{=}1$). Attend-first ($t{=}0$)
demands the growing offset $i$, so by Theorem~\ref{thm:growing} no $\kappa$ \emph{with a strict unique
global maximizer} reaches $\alphaL\to1$. This is the right qualifier: a monotone increasing kernel such
as $\kappa(d)=d$ has no global maximizer and does realize attend-first at every length (the causal window
truncates at $d{=}i$), so what defeats attend-first in practice is the sign constraint on the canonical
ALiBi slope together with the $L_0$-fit's failure to select such a kernel, not translation invariance by
itself; the selected RoPE kernel decays as $\alphaL=0.80,0.40,0.20,0.12,0.055,0.039$ at
$L=30,40,60,128,256,512$.
\emph{Clamped learned-absolute} $\varphi(i)=v_{\min(i,L_0-1)}$: attend-first is realized by the
length-independent column-$0$-hot kernel (margin $+1$ at all lengths). A fixed relative offset moves the
target across absolute positions into the collapsed range, where all keys $j\ge L_0$ share one feature
and their logits tie exactly, so $\alphaL{=}0$ with margin $0$. For the i.i.d.-extended variant the
corresponding statement is the chance bound of Prop.~\ref{prop:learnedabs}. Growing offsets defeat both.
Verified for the \emph{clamped} object in \texttt{src/spectra.py} (attend-first $1.00$; prev $0.00$ with
margin exactly $0$) and for the \emph{i.i.d.-extended} object in
\texttt{experiments/duality\_abs\_vs\_rel.py} (prev $0.046$, the chance level of
Prop.~\ref{prop:learnedabs}). In \texttt{duality\_abs\_vs\_rel.py} the relative side scores $0.00$ on
attend-first and $1.00$ on prev; the concrete PEs behave as their own families dictate, so RoPE scores
$0.80$ on attend-first at $L{=}30$ and decays with $L$ (Prop.~\ref{prop:duality}(a)), while ALiBi scores
$0.00$ on prev (Prop.~\ref{prop:static}).
Note that \texttt{duality\_abs\_vs\_rel.py}'s attend-first ``absolute'' construction is the unrestricted
kernel $K[i,j]=\mathbf 1[j{=}0]$, i.e.\ object~(1) above, which is why it is reported separately from the
clamped values. \hfill$\square$

\section{Failure modes of $\alphaL$}
\label{app:failuremodes}
\paragraph{Failure modes of $\alphaL$.} We list the five ways $\alphaL$
departs from a trained model's length generalization. They are the precise boundary of the single-head
positional theory, and each is something we observe and characterize.
\begin{enumerate}\setlength\itemsep{1pt}
\item \textbf{Selection gap.} A rescuing kernel $W$ exists, but the $L_0$-fit does not select it, so
$\alphaL<1$ despite realizability in principle (sinusoidal fixed-offset: $\alphaL{=}0.45$ at $L{=}60$;
Prop.~\ref{prop:sin}). \emph{Existence $\neq$ selection.}
\item \textbf{Content rescue.} The target is carried by the \emph{content} channel, outside the
abstraction, so the positional $\alphaL$ under-predicts: LG occurs with $\alphaL{=}0$ (induction,
and ALiBi/NoPE prev/first at depth; App.~\ref{app:content}). This is the binding limitation (S3).
\item \textbf{Depth / multi-head escape.} Multiple heads or layers route around the single-head
impossibility (an anchor head $+$ a content-match head solves copy at depth~$1$), so $\alphaL$ can
under-predict at scale while remaining a \emph{sufficient screen} (\S\ref{sec:exp}).
\item \textbf{Low-frequency relative under-prediction.} Even within the positional channel and on a
\emph{content-neutral} task, $\alphaL$ is not a perfect graded predictor: on content-neutral skip-$k$ a
frequency-scaled (low-frequency) RoPE model generalizes a target the linear-kernel fit deems
unrealizable (Table~\ref{tab:freqsweep}, App.~\ref{app:freq}), so $\alphaL$ under-predicts it. The task
is content-neutral, so the discrepancy is not content rescue in the sense of (2); we have not isolated
the mechanism further, and we report it as an unexplained case where the trained model generalizes and the fitted kernel does not.
\item \textbf{Overprediction by the fitted kernel.} The converse also occurs: the fitted kernel can have
high $\alphaL$ while the trained model selects a different solution and fails to generalize
(\texttt{rope-first} at $L{=}30$, $\alphaL{=}0.80$ against LG $0.13$; \texttt{sinusoidal-first} at
$L{=}80$ on one seed). Theorem~\ref{thm:suff} is unaffected, since it is a statement about a single
realizing kernel; what fails is the empirical bridge from $\alpha_L(W_{\mathrm{fit}})$ to the behaviour
of $W_{\mathrm{trained}}$.
\end{enumerate}
\noindent The five are distinct: (1) is intra-positional selection, (2) is a different channel, (3) is a
larger architecture, (4) is a graded-versus-binary gap, and (5) is a mismatch between the fitted and the
trained kernel. The first four leave the \emph{sufficiency} direction ($\alphaL{\to}1\Rightarrow$ LG)
intact and are failures of \emph{necessity}, which is why we call $\alphaL$ a one-sided screen; (5) is
the one case where high $\alphaL$ does not co-occur with length generalization, which is why we report
the empirical link as a co-occurrence rate.

\subsection{Scope of the individual statements}
\label{app:notclaim}
The following limits attach to specific results and are easy to over-read.
(i)~No general sinusoidal LG impossibility:
Conjecture~\ref{conj:general} is open and the intuition that a sinusoidal encoding must fail once the
required offset exceeds its lowest-frequency period is \emph{refuted} (Prop.~\ref{prop:sin}).
(ii)~Sinusoidal fixed-offset transfer is not exact: existence of a rescuing
$W$ does not imply its selection by the $L_0$-fit. (iii)~The learned-absolute rate $1/(k-1)$ is
\emph{conditional} on the argmax falling among the off-diagonal new keys, and applies when the target is
one of them; the self-key is a non-exchangeable quadratic competitor, which is why it is $1/(k-1)$ and
not $1/k$, and the unconditional rate is smaller. (iv)~Prop.~\ref{prop:static}
is about the \emph{unique positive-margin} argmax; set-containment is a different statement (which fails on the measure-zero
tie locus). (v)~All impossibility results are single-head, positional-only; multi-head/depth/content
provably route around them. (vi)~Prop.~\ref{prop:learnedabs} proves chance via i.i.d.\ \emph{random}
untrained features, whereas our learned-absolute implementation \emph{clamps} $\varphi(i){=}e_{\min(i,L_0-1)}$
(matching the trained model); both routes give chance, by exchangeability and by degeneracy respectively,
but they are distinct mechanisms and we do not conflate them.

\section{Experimental details}
\label{app:exp}
\textbf{Architecture.} Decoder-only transformer, $d_{\text{model}}=64$, GELU MLP (width $4d$), pre-LN,
learned token embedding (vocab $12$). PEs: NoPE; sinusoidal (additive); learned-absolute (table of size
$L_0$, positions $\ge L_0$ clamped to $L_0-1$); RoPE (rotary on $q,k$); ALiBi (fixed geometric slopes).
Tiers: 1-layer/1-head, 2-layer/4-head, 4-layer/4-head. \textbf{Training.} AdamW (lr $3\times10^{-4}$,
wd $0.01$), batch $64$, $2000$ steps, variable train length $n\sim\mathrm{Unif}[L_0/2,L_0]$, $L_0=20$,
cross-entropy over query positions (ignore-index elsewhere). \textbf{Tasks} (target position $t_n(i)$,
all position-dependent via small vocab): prev ($i{-}1$), skip-$k$ ($i{-}k$, $k\in\{2,3,5\}$), copy
($i{-}\lceil n/2\rceil$), reverse, first ($0$). \textbf{Evaluation.} Accuracy on extrapolation query
rows $[L_0,L)$ at $L\in\{30,40,50,60,80\}$, $5$ seeds. \textbf{Predictor fit.} Cross-entropy fit of the
kernel coefficients to the $L_0$ targets ($400$ Adam steps), PE-structural constraints imposed
(RoPE relative basis; ALiBi recency slope constrained $\le0$); strict-unique-argmax scoring on new rows.
\textbf{Compute.} Apple M3 Pro (MPS); the full $2625$-measurement, $5$-seed study runs in $\sim$55\,min.

\section{Additional results}
\label{app:add}
Table~\ref{tab:depth} reports the depth trend (H3) over all three tiers; Fig.~\ref{fig:depth} plots it.
As depth grows ($1\to2\to4$ layers), the learnable-cell \emph{graded} rank correlation drops from
$0.90$ (CI $[0.73,0.95]$) to $0.31$ and $0.28$ with bootstrap CIs that include $0$
(Table~\ref{tab:depth}); the precision on high-$\alphaL$ rows instead stays high and tight
($0.925\to0.94\to0.97$). The lost graded signal is explained by the content channel: the set of low-$\alphaL$
cells that nonetheless generalize \emph{grows} with depth. Scoring cells by their means, it is
$\{$ALiBi prev/first/skip2, NoPE prev/first, RoPE first$\}$ at $2$ layers and gains ALiBi skip3 at $4$
layers, with the two marginal members strengthening as depth grows (NoPE prev $0.53\to0.74$, RoPE first
$0.83\to1.00$). Depth does not break the sufficient-indicator
property; it adds more content-channel routes around a low positional $\alphaL$, exactly the (S3)
boundary.

\begin{table}[h]
\centering
\caption{Depth trend (H3), learnable cells ($\mathrm{train\ acc}>0.8$), $5$ seeds. The Spearman column is
\emph{cell-level} with a bootstrap $95\%$ CI over the (PE, task) cells; the precision column is the rate
$\Pr[\mathrm{LG}\ge0.7\mid\alphaL\ge0.5]$ with a bootstrap $95\%$ CI over the high-$\alphaL$ \emph{rows}.
$\alphaL$ stays a high-precision \emph{sufficient} indicator at every depth (precision CIs well above
chance); the graded rank correlation, by contrast, is significant only at depth $1$ -- at depths $2,4$
its CI \emph{includes $0$}, consistent with $\alphaL$ acting as a binary screen rather than a graded
predictor once content rescues low-$\alphaL$ cells.}
\label{tab:depth}
\begin{tabular}{lccc}
\toprule
tier & depth & learnable Spearman $\rho$ (95\% CI) & sufficiency precision (95\% CI) \\
\midrule
1L1H & 1 & 0.90\ \ $[0.73, 0.95]$ & 0.925\ $[0.88, 0.96]$ \\
2L4H & 2 & 0.31\ \ $[-0.16, 0.76]$ & 0.94\ \ $[0.89, 0.97]$ \\
4L4H & 4 & 0.28\ \ $[-0.15, 0.66]$ & 0.97\ \ $[0.94, 0.99]$ \\
\bottomrule
\end{tabular}
\end{table}

\begin{figure}[h]
\centering
\includegraphics[width=0.5\textwidth]{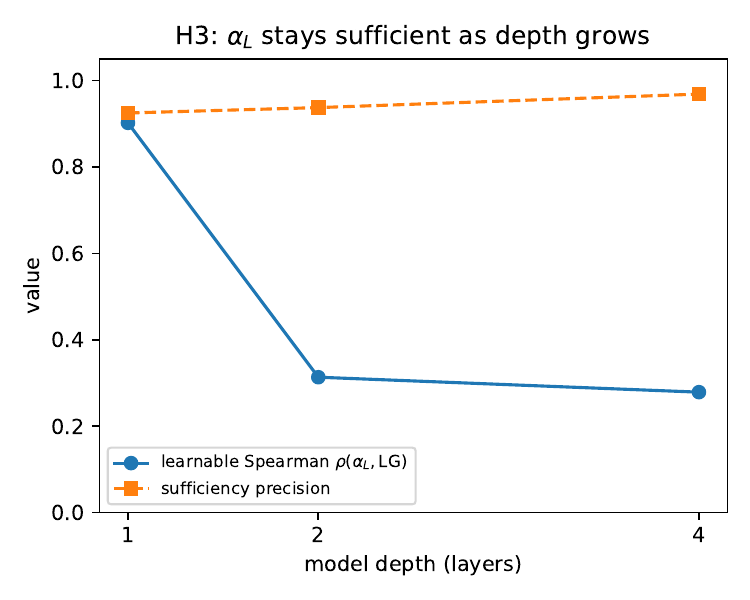}
\caption{H3: learnable-cell Spearman $\rho(\alphaL,\mathrm{LG})$ (circles) drops with depth (its
bootstrap CI includes $0$ at depths $2,4$), while sufficiency precision (squares) stays $\ge0.92$ with
tight CIs. $\alphaL$ remains a \emph{sufficient screen} at all depths.}
\label{fig:depth}
\end{figure}

\paragraph{Inductive-bias ablation (which fit drives the correlation).} Table~\ref{tab:ablation}
computes $\alphaL$ three ways on the $75$ 1L1H $(\text{PE,task},L)$ rows and correlates each
with empirical LG. The min-norm interpolant of Theorem~\ref{thm:suff} yields only $\rho{=}0.56$; the
regularized cross-entropy fit used throughout yields $0.91$, and its weight-decay$\to0$ limit, which
approaches the max-margin direction \citep{soudry2018implicit}, yields $0.92$. The predictive content of
$\alphaL$ therefore comes from the cross-entropy direction rather than from the min-norm object, a gap we
state here.

\begin{table}[h]
\centering
\caption{Inductive-bias ablation ($75$ rows in the 1L1H setting). The cross-entropy fit, not the min-norm
interpolant, drives the $\alphaL$--LG correlation.}
\label{tab:ablation}
\begin{tabular}{lcc}
\toprule
kernel fit & Spearman $\rho$ & Pearson \\
\midrule
min-norm interpolant (Thm~\ref{thm:suff}'s $W^*$) & 0.56 & 0.57 \\
cross-entropy, wd$=10^{-4}$ (reported $\alphaL$) & 0.91 & 0.95 \\
cross-entropy, wd$\to0$ (max-margin) & 0.92 & 0.94 \\
\bottomrule
\end{tabular}
\end{table}

\paragraph{Robustness checks.} The learnability cutoff is inert: cell-level $\rho=0.91/0.90/0.90$ at
$\mathrm{train}>\{0.5,0.7,0.9\}$ (train accuracy is bimodal). On the $20$ dropped cells $\alphaL$ has
std $0.02$ (uniformly $\approx0$), so it provides no rank signal there by construction. Scale robustness
(vocab $64$, $d{=}128$) is reported in App.~\ref{app:scale}; training-length robustness (the screen AUC
stays $1.0$ and the per-cell verdicts are unchanged at $L_0\in\{10,20,40\}$) in App.~\ref{app:l0sens};
and PE frequency/$\theta$-base/dimension robustness in App.~\ref{app:freq}. Interior fixed-absolute
positions, which remove the reliance on the special boundary position $0$, are in App.~\ref{app:interior}.

\section{Interior fixed-absolute positions}
\label{app:interior}
The fixed-absolute regime in the main grid is represented by \texttt{first}, whose target is position
$0$. That position is special: it is the causal boundary, so a monotone recency kernel can reach it
without addressing an absolute position at all. To check that the absolute side of the duality does not
depend on this artefact, we evaluate two interior targets, \texttt{pos2} and \texttt{pos5} (attend to
absolute position $2$, resp.\ $5$), and ran the full study on them: $5$ PEs $\times$ $2$ tasks $\times$
$3$ depth tiers $\times$ $5$ seeds $\times$ $5$ lengths, $750$ further trained-model measurements.

The separation is if anything cleaner than at the boundary (Table~\ref{tab:interior}). The absolute
encodings solve both interior positions and transfer them perfectly ($\alpha_{60}=1.00$, LG $\ge0.97$ at
$L=80$), while the relative encodings neither learn nor transfer them. The screen AUC over these ten
cells is $1.00$, with $\alphaL=1.00$ on every generalizing cell and $0.00$ on every failing one. The
boundary artefact is visible in the comparison: RoPE scores $\alpha_{60}=0.20$ on \texttt{first} but
$0.00$ on both interior positions, i.e.\ the residual score at position $0$ came from the causal boundary
rather than from any ability to address absolute positions.

\begin{table}[H]
\centering\small
\caption{Interior fixed-absolute positions (1L1H, mean over $5$ seeds). Absolute encodings realize and
transfer them; relative encodings do not learn them at all. Compare \texttt{first}, where RoPE scores
$\alpha_{60}=0.20$ purely because position $0$ is the causal boundary.}
\label{tab:interior}
\begin{tabular}{@{}llcccc@{}}
\toprule
PE & task & $\alpha_{60}$ & train acc & LG@60 & LG@80 \\
\midrule
learned-abs & pos2 & 1.00 & 1.00 & 1.00 & 1.00 \\
learned-abs & pos5 & 1.00 & 1.00 & 1.00 & 1.00 \\
sinusoidal  & pos2 & 1.00 & 0.98 & 0.98 & 0.97 \\
sinusoidal  & pos5 & 1.00 & 1.00 & 1.00 & 0.99 \\
\midrule
NoPE  & pos2 & 0.00 & 0.24 & 0.16 & 0.16 \\
NoPE  & pos5 & 0.00 & 0.22 & 0.16 & 0.15 \\
ALiBi & pos2 & 0.00 & 0.23 & 0.16 & 0.15 \\
ALiBi & pos5 & 0.00 & 0.22 & 0.16 & 0.15 \\
RoPE  & pos2 & 0.00 & 0.31 & 0.15 & 0.14 \\
RoPE  & pos5 & 0.00 & 0.20 & 0.16 & 0.15 \\
\bottomrule
\end{tabular}
\end{table}

\section{Scale robustness}
\label{app:scale}
Because $\alphaL$ depends only on $(\Phimat,t_n)$, it is \emph{identical} across model scales; the
question is whether the empirical LG it predicts is stable. We re-ran the full study at a $5\times$ larger
vocabulary ($64$ vs $12$) and $2\times$ width ($d{=}128$ vs $64$). The 1L1H cell-level
correlation is \emph{unchanged} (indeed slightly higher): $\rho=0.92$ (vs $0.90$ at $d{=}64$), with
sufficiency precision $0.94$; the scale tier moves $0.36\!\to\!0.45$ (Table~\ref{tab:scale}). The
headline relationship is therefore not an artifact of the tiny vocabulary or width. (This is a
robustness check within the synthetic position-driven family, not a claim about natural-language LLMs;
see \S\ref{sec:limits}.)

\begin{table}[h]
\centering
\caption{Scale robustness: cell-level Spearman $\rho(\alphaL,\mathrm{LG})$ on learnable cells, at the
paper's default size vs.\ a larger model. The predictor is identical; only the trained models change. The
scale study uses $2$ seeds ($700$ rows), and both columns are computed on the cells the two studies have
in common, so the $2$L$4$H entry ($0.36$) is over $19$ cells rather than the $20$ of
Table~\ref{tab:depth}, where the same quantity over all $20$ cells is $0.31$.}
\label{tab:scale}
\begin{tabular}{lcc}
\toprule
tier & vocab $12$, $d{=}64$ (main) & vocab $64$, $d{=}128$ \\
\midrule
1L1H & 0.90 & 0.92 \\
2L4H (scale) & 0.36 & 0.45 \\
\bottomrule
\end{tabular}
\end{table}

\section{Training-length ($L_0$) sensitivity}
\label{app:l0sens}
Is the clean separation an artifact of the default training length $L_0{=}20$? We re-ran the
1-layer 1-head study at $L_0\in\{10,20,40\}$, each evaluated at $3\times L_0$
(so the extrapolation factor is held fixed), $3$ seeds. The screen AUC is $1.000$ at every $L_0$, and
\emph{the same four cells generalize} in each case (rope prev/skip-3, sinusoidal/learned-abs first); the
Theorem~\ref{thm:growing} verdicts are unchanged (relative bounded offsets transfer; relative growing
offsets do not). The separation is therefore not specific to $L_0{=}20$ (Table~\ref{tab:l0sens}).

\begin{table}[H]
\centering\small
\caption{$L_0$-sensitivity (1L1H, predictor $\alphaL$ / empirical LG at $L{=}3L_0$). The screen AUC is
$1.0$ at every training length and the per-cell verdicts are stable; the separation is not an artifact of
$L_0{=}20$. Generated by \texttt{experiments/l0\_sensitivity.py}.}
\label{tab:l0sens}
\begin{tabular}{@{}lccc@{}}
\toprule
 & $L_0{=}10\ (\to30)$ & $L_0{=}20\ (\to60)$ & $L_0{=}40\ (\to120)$ \\
\midrule
screen AUC                       & \textbf{1.00} & \textbf{1.00} & \textbf{1.00} \\
\# generalizing cells            & 4/25 & 4/25 & 4/25 \\
\midrule
rope / prev ($\alphaL$\,/\,LG)   & 1.00\,/\,0.93 & 1.00\,/\,0.87 & 1.00\,/\,0.87 \\
rope / skip-3                    & 1.00\,/\,1.00 & 1.00\,/\,1.00 & 1.00\,/\,0.98 \\
rope / copy (growing)            & 0.00\,/\,0.12 & 0.00\,/\,0.10 & 0.00\,/\,0.09 \\
rope / reverse (growing)         & 0.00\,/\,0.16 & 0.03\,/\,0.14 & 0.00\,/\,0.12 \\
sinusoidal / first (absolute)    & 1.00\,/\,0.92 & 1.00\,/\,0.83 & 1.00\,/\,0.88 \\
learned-abs / first (absolute)   & 1.00\,/\,1.00 & 1.00\,/\,1.00 & 1.00\,/\,1.00 \\
\bottomrule
\end{tabular}
\end{table}

\section{A content-aware extension: $\alpha_L^{+}$ (with a demonstration)}
\label{app:content}
The binding limitation is (S3): where the content channel carries the target, the positional $\alphaL$
under-predicts. We show the one-sidedness is closable, using a content-driven task, and state what is
and is not recovered.

\paragraph{Setup.} Augment the realizable space from the positional kernel $\Kn_n$ to
$\Kn_n + \mathcal{C}$, where $\mathcal{C}$ is the content-matching channel
$e(x_i)^\top W_c e(x_j)$. Define $\alpha_L^{+}=\max(\alphaL,\ \alpha_L^{\mathrm{content}})$, where
$\alpha_L^{\mathrm{content}}=1$ iff the target is realizable by a \emph{length-stationary} content rule
(a token match composed with a fixed shift). The canonical such task is \textbf{induction}: the target
at $i$ is the token following the previous occurrence of $x_i$ -- content-determined, hence
$\alphaL{=}0$ for every PE (no length-dependent positional rule realizes it), while
$\alpha_L^{\mathrm{content}}{=}1$.

\paragraph{Demonstration (Table~\ref{tab:induction}).} We train induction across PEs.
At $1$ layer the circuit cannot form (train acc $\approx0.25$); at $2$ layers it is learned (train
$0.76$--$1.0$) and \emph{generalizes}, strongly for ALiBi (LG $=0.84$), partially for RoPE/NoPE
($0.43/0.34$). So positional $\alphaL=0$ \emph{mispredicts} a cell with LG~$0.84$ -- a real
one-sidedness failure -- which $\alpha_L^{+}$ correctly flags as realizable through the content channel. \textbf{Caveat:} induction LG is PE-\emph{modulated} ($0.26$--$0.84$: the content circuit still needs a working
prev-token head, which absolute PEs support poorly), so the \emph{binary} $\alpha_L^{\mathrm{content}}$
over-predicts for sinusoidal/learned-abs. A graded content channel weighting $\alpha_L^{\mathrm{content}}$
by the PE's support for the match circuit -- and the multi-head extension that provably escapes
Theorem~\ref{thm:growing} (anchor head $+$ content-match head solves copy at depth~1) -- are the natural
next steps. The point is narrow: the positional indicator's blind spot is identifiable, and the content channel
recovers a cell it misses.

\paragraph{A frequency sweep for a discriminating cell.} One may reasonably ask for one
(PE, task) cell where $\alphaL$ beats the relative/absolute taxonomy. We searched the most promising
direction: a \emph{frequency-scaled RoPE} (still relative) on the \emph{content-neutral} skip-$k$ task
(the query token is independent of the target $s[i{-}k]$, so content cannot supply the answer), sweeping
the frequency scale to provoke a too smooth to separate nearby offsets failure $\alphaL$ would catch. The result is
negative (Table~\ref{tab:freqsweep}): standard and moderately-low frequencies resolve the offset
($\alphaL{=}1$, LG~$\approx1$); at scale $0.03$ the trained head \emph{still} generalizes (LG~$=0.88$)
while $\alphaL{=}0$ -- an $\alphaL$ \emph{false negative}, with the trained model generalizing a target
the linear-kernel fit deems unrealizable; and at scale $0.01$ the model fails to learn even
in-distribution (excluded by the learnability filter). \textbf{No discriminating cell exists in this
family, and the probe instead reveals that $\alphaL$ under-predicts low-frequency relative PEs.}
Combined with the row-level over-predictions (rope-first), the verdict is that $\alphaL$
\emph{matches but does not beat} the taxonomy as a predictor; its contribution is the theory that
\emph{explains} the taxonomy, not a superior predictor. (The low-frequency case is not content rescue, since skip-$k$ is content-neutral; we do not claim to
have identified the mechanism behind it.)

\begin{table}[H]
\centering
\caption{Content-driven induction (target $=$ token after the previous occurrence of $x_i$): positional
$\alphaL=0$ for all PEs, yet $2$-layer models generalize (LG up to $0.84$). $\alpha_L^{+}$ flags the
realizable through the content channel regime $\alphaL$ misses; residual PE-modulation of content LG is future work.}
\label{tab:induction}
\begin{tabular}{lccc}
\toprule
PE & train acc (2L4H) & LG@60 (2L4H) & positional $\alphaL$ \\
\midrule
ALiBi       & 0.96 & \textbf{0.84} & 0 \\
RoPE        & 1.00 & 0.43 & 0 \\
NoPE        & 0.76 & 0.34 & 0 \\
learned-abs & 0.99 & 0.27 & 0 \\
sinusoidal  & 0.98 & 0.26 & 0 \\
\bottomrule
\end{tabular}
\end{table}

\begin{table}[H]
\centering
\caption{RoPE frequency sweep on content-neutral skip-3 (1L1H). No discriminating cell: where the model
generalizes the taxonomy is right; at scale $0.03$ $\alphaL$ under-predicts (false negative); at
$0.01$ the model cannot learn at all.}
\label{tab:freqsweep}
\begin{tabular}{lccc}
\toprule
RoPE freq.\ scale & $\alphaL$ & train acc & LG@60 \\
\midrule
$1.0$ / $0.3$ & 1.00 & 1.00 & 1.00 \\
$0.1$         & 1.00 & 0.94 & 0.93 \\
$0.03$        & \textbf{0.00} & 0.88 & \textbf{0.88} \\
$0.01$        & 0.00 & 0.23 & 0.18 \\
\bottomrule
\end{tabular}
\end{table}

\section{Frequency and dimension sensitivity of $\alphaL$}
\label{app:freq}
Since $\alphaL$ is computed from the PE feature map, it is fair to ask whether its verdict is an artifact
of our default frequency count, $\theta$-base, or feature dimension. We sweep each (predictor-only, no
training; $L_0{=}20\!\to\!L{=}60$) and find a clean separation: the realizability \emph{verdict} is
stable wherever the theory makes a clean claim, and the only cells whose realizability verdict
moves are exactly the sinusoidal bounded-offset cells the theory flags as a selection gap
(Prop.~\ref{prop:sin}); a few other cells show small sub-threshold raw movement that never flips the verdict.

\begin{table}[H]
\centering\small
\caption{$\alphaL$ for RoPE under varying frequency count $\mathrm{num}$ (feature dimension $m{=}2\,\mathrm{num}$)
and $\theta$-base. The verdict is invariant: fixed relative offsets transfer ($\mathrm{prev,skip3}{=}1.00$
everywhere), growing offsets never transfer (copy/reverse${\approx}0$), and the kernel the $L_0$-fit
selects does not hit a fixed absolute position (first stays at chance; a monotone kernel in the family
could, but the fit does not select it, see the discussion after Prop.~\ref{prop:duality}). For
low-frequency \emph{scaling} see Table~\ref{tab:freqsweep}.}
\label{tab:freqrope}
\begin{tabular}{@{}lccccc@{}}
\toprule
RoPE setting & prev & skip3 & copy & reverse & first \\
\midrule
$\mathrm{num}{=}8$ ($m{=}16$)   & 1.00 & 1.00 & 0.00 & 0.00 & 0.17 \\
$\mathrm{num}{=}16$ ($m{=}32$)  & 1.00 & 1.00 & 0.00 & 0.03 & 0.20 \\
$\mathrm{num}{=}32$ ($m{=}64$)  & 1.00 & 1.00 & 0.00 & 0.00 & 0.03 \\
base $10^{3}$                   & 1.00 & 1.00 & 0.00 & 0.03 & 0.05 \\
base $10^{5}$                   & 1.00 & 1.00 & 0.00 & 0.03 & 0.03 \\
\bottomrule
\end{tabular}
\end{table}

\begin{table}[H]
\centering\small
\caption{$\alphaL$ for sinusoidal under the same sweep. The absolute verdict is invariant (first${\approx}1$
everywhere; growing offsets ${\approx}0$). The bounded-offset cells (prev/skip3) are the only
entries whose verdict moves with frequency -- they \emph{degrade} as frequencies are added
($\mathrm{num}{:}8{\to}32$ drives prev $0.47{\to}0.00$); the other cells' raw $\alphaL$ shifts stay
sub-threshold (e.g.\ first $0.93{\to}1.00$) and never flip the verdict. This is what the theory predicts: enlarging the realizable space does not help the $L_0$-fit
\emph{select} a transferring kernel (existence $\neq$ selection, Prop.~\ref{prop:sin}). Verdict movement
appears precisely at the selection gap and nowhere else.}
\label{tab:freqsin}
\begin{tabular}{@{}lccccc@{}}
\toprule
sinusoidal setting & prev & skip3 & copy & reverse & first \\
\midrule
$\mathrm{num}{=}8$ ($m{=}16$)   & 0.47 & 0.50 & 0.00 & 0.00 & 0.93 \\
$\mathrm{num}{=}16$ ($m{=}32$)  & 0.45 & 0.25 & 0.07 & 0.00 & 1.00 \\
$\mathrm{num}{=}32$ ($m{=}64$)  & 0.00 & 0.00 & 0.00 & 0.00 & 1.00 \\
base $10^{3}$                   & 0.17 & 0.28 & 0.00 & 0.00 & 1.00 \\
base $10^{5}$                   & 0.15 & 0.12 & 0.00 & 0.00 & 1.00 \\
\bottomrule
\end{tabular}
\end{table}

\noindent\textbf{Takeaway.} Across frequency count, $\theta$-base, and feature dimension, $\alphaL$'s
\emph{binary realizability verdict} is stable for every regime with a clean theoretical prediction; the
residual frequency dependence is confined to the sinusoidal bounded-offset selection gap, which is the
behaviour Prop.~\ref{prop:sin} predicts. The screen is therefore not an artifact of the default PE
hyperparameters. (Generated by \texttt{experiments/freq\_dim\_sensitivity.py}.)

\section{Reproducibility checklist}
\label{app:repro}
The full source (predictor, models, tasks, experiments and theory-verification scripts) is
released at \url{https://github.com/superajgklasjglaks/positional-realizability}.
\begin{itemize}\setlength\itemsep{1pt}
\item \textbf{Predictor.} $\alphaL$ is in \texttt{src/spectra.py} (\texttt{simulate(pe,task,$L_0$,$L$)});
it uses no end-to-end Transformer training, no data samples, and no learned Transformer parameters -- only a
deterministic kernel-selection on the PE basis. Models: \texttt{src/model.py} (the five PEs);
tasks: \texttt{src/data.py}.
\item \textbf{Seeds.} The main study uses seeds $\{0,1,2,3,4\}$ (fixed in \texttt{run\_study.py});
bootstrap CIs resample cells for the correlations and high-$\alphaL$ rows for the precisions. All randomness is seeded.
\item \textbf{Configuration.} The design grid (PEs $\times$ tasks $\times$ lengths $\times$ seeds $\times$
depth tiers) is Table~\ref{tab:design}; architecture/optimizer/training details are in
App.~\ref{app:exp}. Tiers and tags are env-overridable in \texttt{run\_study.py}
(\texttt{SEEDS}, \texttt{DMODEL}, \texttt{TIERS\_ONLY}, \texttt{STUDY\_TAG}); the run is resumable.
\item \textbf{Data behind every number.} \texttt{results/study\_rows.json} (all $2625$ rows, including
runs the model failed to learn at $L_0$ -- no failed runs are silently dropped; the learnability
filter \texttt{train\_acc}${>}0.8$ is applied explicitly and reported alongside the unconditioned
numbers). Summaries in \texttt{results/study\_summary.json}.
\item \textbf{Figures/tables $\to$ scripts.} Figs.~\ref{fig:scatter}--\ref{fig:depth} and the heatmaps:
\texttt{make\_figures.py}. Headline statistics (AUC, permutation $p$, Spearman + bootstrap CIs,
sufficiency precision): \texttt{review\_stats.py}. Scale robustness (Table~\ref{tab:scale}):
\texttt{run\_study.py} with \texttt{STUDY\_TAG=scale}. $L_0$-sensitivity (Table~\ref{tab:l0sens}):
\texttt{l0\_sensitivity.py}. Frequency/dimension sensitivity (Tables~\ref{tab:freqrope}--\ref{tab:freqsin}):
\texttt{freq\_dim\_sensitivity.py}. Content extension / induction (Table~\ref{tab:induction}) and the
RoPE frequency sweep (Table~\ref{tab:freqsweep}): \texttt{ablation\_fit.py} and the content study.
\item \textbf{Theory verification.} The following claims have executable checks; the list is
exhaustive, and claims not named here are supported by their proofs alone:
relative growing-offset impossibility (\texttt{theorem\_relative\_copy\_impossibility.py}), the
deep-row iff and its counterexample $\kappa{=}[1.46,1.96,1.80,6.32]$ (\texttt{iff\_constant\_offset\_lg.py}),
learned-absolute chance $1/(k{-}1)$ (\texttt{verify\_learned\_abs\_fixed\_position\_iff.py},
\texttt{theorem2b\_learned\_abs\_montecarlo.py}), the absolute/relative duality
(\texttt{duality\_abs\_vs\_rel.py}), and the Conjecture~\ref{conj:general} cone/selection probes
(\texttt{underdetermined\_conj1\_*.py}).
\item \textbf{Compute.} The full study runs in ${\approx}55$ min on one Apple M3 Pro (MPS); the predictor
and all sensitivity sweeps run in seconds on CPU.
\end{itemize}

\end{document}